\documentclass[dvipsnames]{article}
\pdfoutput=1

\usepackage{microtype}
\usepackage{graphicx}
\usepackage{booktabs} 

\usepackage{hyperref}

\usepackage[accepted]{icml2020}

\usepackage{natbib}
\usepackage{amsmath,amssymb,amsthm,amsfonts}
\usepackage{algorithmic}
\usepackage{graphicx}
\usepackage{url}
\usepackage{multicol}
\usepackage{subfig}
\usepackage[algo2e,ruled,noend,linesnumbered]{algorithm2e}
\usepackage{enumitem}
\usepackage{booktabs}							
\usepackage{array}								
\usepackage{bbding} 							
\usepackage{bm,bbm}
\usepackage{xspace}
\usepackage{float}								

\usepackage{tikz}
\usetikzlibrary{positioning,calc}
\usepackage{xcolor}
\definecolor{cRed}{HTML}{DA5527}
\definecolor{cGold}{HTML}{EEB11D}
\definecolor{cBlue}{HTML}{0A73B9}
\definecolor{cGreen}{HTML}{008D0A}
\definecolor{cPurple}{HTML}{990099}

\usepackage[final]{changes} 			
\newcommand{\inlinecomment}[1]{\deleted{#1}}

\newcommand{\somesh}[1]{\inlinecomment{Somesh: #1}}

\theoremstyle{plain}

\newtheorem{prop}{Proposition}

\theoremstyle{definition}
\newtheorem{defn}{Definition}

\newtheorem*{rem*}{Remark}

\newcommand{\hide}[1]{}

\def\e{\mathbf{e}}

\def\h{\mathbf{h}}

\def\v{\mathbf{v}}

\def\x{\mathbf{x}}

\def\z{\mathbf{z}}

\def\A{\mathbf{A}}

\def\C{\mathbf{C}}

\def\N{\mathbf{N}}

\def\V{\mathbf{V}}
\def\W{\mathbf{W}}
\def\X{\mathbf{X}}
\def\Y{\mathbf{Y}}

\def\1{\mathbf{1}}
\def\0{\mathbf{0}}

\def\balpha{\bm{\alpha}}

\def\btau{\bm{\tau}}

\def\cA{\mathcal{A}}
\def\cB{\mathcal{B}}

\def\cF{\mathcal{F}}

\def\cH{\mathcal{H}}
\def\cI{\mathcal{I}}

\def\cK{\mathcal{K}}

\def\cN{\mathcal{N}}

\def\cR{\mathcal{R}}
\def\cS{\mathcal{S}}

\def\cX{\mathcal{X}}

\def\bbN{\mathbb{N}}
\def\bbR{\mathbb{R}}
\def\bbE{\mathbb{E}}
\def\bbI{\mathbbm{I}}
\def\bb1{\mathbbm{1}}

\def\tdY{\tilde{Y}}
\def\ulx{\underline{\x}}
\def\ulX{\underline{\X}}

\def\defas{\triangleq}

\def\pluseq{{\mathrel{+}=}}

\def\Pois{\mathrm{Poisson}}

\def\Uniform{\mathrm{Uniform}}
\def\Exp{\mathrm{Exp}}

\def\IG{\mathrm{IG}}
\def\DeepLIFT{\mathrm{DeepLIFT}}
\def\Attr{A}
\def\SeqEnc{\mathrm{Enc}}

\def\method{CAUSE\xspace} 

\def\Excitation{Excitation\xspace}
\def\Inhibition{Inhibition\xspace}
\def\Synergy{Synergy\xspace}
\def\IPTV{IPTV\xspace}
\def\MT{MT\xspace}

\usepackage{enumitem}

\allowdisplaybreaks 
\setlist{noitemsep, topsep=0pt,parsep=0pt,partopsep=0pt, leftmargin=15pt}
\newcommand{\reducemargin}{\vspace*{0ex}} 
\newcommand{\myparagraph}[1]{\vspace*{-1.5ex}\paragraph{#1}}
\newcommand{\mysection}[1]{\vspace*{0ex}\section{#1}\vspace*{0ex}}
\newcommand{\mysubsection}[1]{\vspace*{0ex}\subsection{#1}}
\newcommand{\mysubsubsection}[1]{\vspace*{0ex}\subsubsection{#1}\vspace*{0ex}}

\begin{document}


\twocolumn[

\reducemargin
\icmltitle{\method: Learning Granger Causality from Event Sequences using Attribution Methods}




\begin{icmlauthorlist}
\icmlauthor{Wei Zhang}{uw}
\icmlauthor{Thomas Kobber Panum}{aau}
\icmlauthor{Somesh Jha}{uw,xp}
\icmlauthor{Prasad Chalasani}{xp}
\icmlauthor{David Page}{duke}
\end{icmlauthorlist}

\icmlaffiliation{uw}{Computer Scineces Department, University of Wisconsin-Madison, Madison, WI, USA.}
\icmlaffiliation{aau}{Department of Electronic Systems, Aalborg University, Aalborg, Denmark.}
\icmlaffiliation{xp}{XaiPient, Princeton, NJ, USA.}
\icmlaffiliation{duke}{Department of Biostatistics and Bioinformatics, Duke University, Durham, NC, USA.}

\icmlcorrespondingauthor{Wei Zhang}{zhangwei@cs.wisc.edu}

\icmlkeywords{Machine Learning, ICML}

\vskip 0.3in
]



\printAffiliationsAndNotice{}  

\begin{abstract}
  We study the problem of learning Granger causality between event types from asynchronous, interdependent, multi-type event sequences.
Existing work suffers from either limited model flexibility or poor model explainability and thus fails to uncover Granger causality across a wide variety of event sequences with diverse event interdependency.
To address these weaknesses, we propose \method (\underline{C}ausality from \underline{A}ttrib\underline{U}tions on \underline{S}equence of \underline{E}vents), a novel framework for the studied task.
The key idea of \method is to first implicitly capture the underlying event interdependency by fitting a neural point process, and then extract from the process a Granger causality statistic using an axiomatic attribution method.
Across multiple datasets riddled with diverse event interdependency, we demonstrate that \method achieves superior performance on correctly inferring the inter-type Granger causality over a range of state-of-the-art methods.
\end{abstract}

\mysection{Introduction}
\label{sec:intro}
Many real-world processes generate a massive number of asynchronous and interdependent events in real time.
Examples include the diagnosis and drug prescription histories of patients in electronic health records, the posting and responding behaviors of users on social media, and the purchase and selling orders executed by traders in stock markets, among others.
Such data can be generally viewed as \emph{multi-type event sequences}, in which each event consists of both a timestamp and
a type label, indicating when and what the event is, respectively.

In this paper, we focus on the fundamental problem of uncovering causal structure among event types from multi-type event sequence data.
Since the question of ``true causality'' is deeply philosophical \citep{schaffer2003metaphysics}, and there are still massive debates on its definition \citep{pearl2009causality,imbens2015causal}, we consider a weaker notion of causality based on predictability---Granger causality.
While Granger causality was initially used for studying the dependence structure for multivariate time series \citep{Granger1969,Dahlhaus2003}, it has also been extended to multi-type event sequences \citep{Didelez2008}.
Intuitively, for event sequence data, an event type is said to be (strongly) \emph{Granger causal} for another event type if the inclusion of historical events of the former type leads to better predictions of future events of the latter type.

Due to their asynchronous nature, in the literature, multi-type event sequences are often modeled by multivariate point process (MPP), a class of stochastic processes that characterize the random generation of points on the real line.
Existing point process models for inferring inter-type Granger causality from multi-type event sequences appear to be limited to a particular case of MPPs---Hawkes process \citep{Eichler2017,xu2016learning,achab2017uncovering}, which assumes past events can only independently and additively \emph{excite} the occurrence of future events according to a collection of pairwise kernel functions.
While these Hawkes process-based models are very interpretable and many include favorable statistical properties, the strong parametric assumptions inherent in Hawkes processes render these models unsuitable for many real-world event sequences with potentially abundant \emph{inhibitive} effects or event \emph{interactions}.
For example, maintenance events should reduce the chances of a system breaking down, and a patient who takes multiple medicines at the same time is more likely to experience unexpected adverse events.

Regarding event sequence modeling in general, a new class of MPPs, loosely referred to as neural point processes (NPPs), has recently emerged in the literature \citep{du2016recurrent,xiao2017modeling,mei2017neuralhawkes,Xiao2019}.
NPPs use deep (mostly recurrent) neural networks to capture complex event dependencies, and thus excel at predicting future events due to their model flexibility.
However, NPPs are uninterpretable and unable to provide insight into the Granger causality between event types.

To address this tension between model explainability and model flexibility in existing point process models, we propose \method (\underline{C}ausality from \underline{A}ttrib\underline{U}tions on \underline{S}equence of \underline{E}vents), a framework for obtaining Granger causality from multi-type event sequences using information captured by a highly predictive NPP model.
At the core of \method are two steps: first, it trains a flexible NPP model to capture the complex event interdependency, then it computes a novel Granger causality statistic by inspecting the trained NPP with an axiomatic attribution method.
In this way, \method is the first technique that brings model-agnostic explainability to NPPs and endows NPPs with the ability to discover Granger causality from multi-type event sequences exhibiting various types of event interdependencies.

\myparagraph{Contributions.}

Our contributions are:
\begin{itemize}
 \item We bring model-agnostic explainability to NPPs and propose \method, a novel framework for learning Granger causality from multi-type event sequences exhibiting various types of event interdependency.
 \item We design a two-level batching algorithm that enables efficient computation of Granger causality scalable to millions of events.
 \item We evaluate \method on both synthetic and real-world datasets riddled with diverse event interdependency. Our experiments demonstrate that \method outperforms other state-of-the-art methods.
\end{itemize}


\myparagraph{Reproducibility.}

We publish our data and our code at
\url{https://github.com/razhangwei/CAUSE}.

\mysection{Background}
\label{sec:background}
In this section, we first establish some notation and then briefly introduce the background for several highly relevant topics.
\mysubsection{Notations}

Suppose there are $S$ subjects and each subject $s$ is associated with a multi-type event sequence $\{(t^s_i, k^s_i)\}_{i = 1}^{n_s}$, where $t^s_i \in \bbR_+$ is the timestamp of the $i$-th event of the $s$-th sequence, $k^s_i \in [K]$ is the corresponding event type, and $n_s$ is the sequence length.
We denote by $\z^s_i \in \{0, 1\}^K$ the one-hot representation of each event type $k^s_i$, and use $[n]$ to represent the set $\{1,\ldots, n\}$ for any positive integer $n$.
To avoid clutter, we omit the subscript/superscript of index $s$ when we are discussing a single event sequence and no confusion arises.

\mysubsection{Multivariate Point Process}
\label{subsec:mpp}

\emph{Multivariate point processes (MPPs)} \citep{Daley2003} are a particular class of stochastic processes that characterize the dynamics of discrete events of multiple types in continuous time.
The most common way to define an MPP is through a set of \emph{conditional intensity functions (CIFs)}, one for each event type.
Specifically, let $N_k(t) \defas \sum_{i = 1}^\infty \bb1(t_i \leq t \wedge k_i = k)$ be the number of events of type $k$ that have occurred up to $t$, and let $\cH(t) \defas \{(t_i, k_i) | t_i < t\}$ be the \emph{history} of all types of events before $t$.
The CIF for event type $k$ is defined as the expected instantaneous event occurrence rate conditioned on history, i.e.,
\begin{equation*}
 \lambda^*_k(t) \defas \lim_{\Delta t \downarrow 0} \frac{\bbE [N_k(t + \Delta t) - N_k (t) | \cH(t)]}{\Delta t},
\end{equation*}
where the use of the asterisk is a notational convention to emphasize that intensity is conditioned upon $\cH(t)$.

Different parameterizations of CIFs lead to different MPPs.
One classic example of MPP is the multivariate Hawkes process \citep{hawkes1971point,hawkes1971spectra}, which assumes each $\lambda^*_k(t)$ to be of the following form:
\begin{equation}
 \lambda^*_k(t) = \mu_k + \sum_{i: t_i < t} \phi_{k, k_i} (t - t_i),
 \label{eq:hawkes-cif}
\end{equation}
where $\mu_k \in \bbR_+$ is the baseline rate for event type $k$, and $\phi_{k, k'}(\cdot)$ for any $k, k' \in [K]$ is a non-negative-valued function (usually referred to as \emph{kernel function}) that characterizes the excitation effect of event type $k'$ on type $k$.

More recently, a class of MPPs loosely referred to as \emph{neural point processes} have emerged in the literature \citep{du2016recurrent,xiao2017modeling,mei2017neuralhawkes,Xiao2019}. These models parameterize CIFs with deep neural networks and generally follow an encoder-decoder design: an \emph{encoder} successively embeds the event history $\{(t_j, k_j) \}_{j=1}^i$ into a vector $\h_i \in \bbR^{N_h}$ for each $i$,
and a \emph{decoder} then predicts with $\h_i$ the future CIFs $\lambda^*_k(t)$ after time $ t_i $ (until the next event is generated).

Most MPPs are trained by minimizing the negative log-likelihood (NLL):
\begin{equation}
 \sum_{s=1}^S \sum_{i = 1}^{n_s} \left[- \log \lambda^{*s}_{k^s_i}(t^s_i) + \sum_{k = 1}^K \int_{t^s_{i - 1}}^{t^s_i} \lambda^{*s}_{k}(t') dt' \right],
 \label{eq:nll-mpp}
\end{equation}
where $\lambda^{*s}_k(t) \defas \lambda^*_k(t | \cH_s(t))$ is the CIF of the $s$-th sequence for the type $k$.
In~\eqref{eq:nll-mpp}, the first term corresponds to the NLL of an event of type $k^s_i$ being observed at $t^s_i$ for the $s$-th sequence, whereas the second term is the NLL of the observation that no events of any type occurred during the window $(t^s_{i - 1}, t^s_{i})$.
When there are no closed-form expressions for the integrals $\int_{t^s_{i - 1}}^{t^s_i} \lambda^{*s}_{k}(t') dt'$, Monte-Carlo approximation needs to be used to approximate either the integrals themselves or their gradients with respect to the parameters. However, these approximation techniques are inefficient and generally suffer from large variances, resulting in low convergence rate.

\mysubsection{Granger Causality for Multi-Type Event Sequences}
\label{subsec:granger-causality-mpp}

The Granger causality definition for multi-type event sequences is established based on point process theory \citep{Daley2003}.
To proceed formally, for any $\cK \subseteq [K]$, we denote by $\cH_{\cK}(t)$ the natural filtration expanded by the sub-process ${\{N_k(t)\}}_{k \in \cK}$; that is, the sequence of smallest $\sigma$-algebra expanded by the past event history of any type $k \in \cK$ and $t \in \bbR_+$, i.e., $\cH_{\cK}(t) = \sigma( \{N_k(s) | k \in \cK, s < t\})$.\footnote{Here, we abuse our previous notation $\cH(t)$ that denotes the set of events that occurred prior to $t$. Appendix~\ref{ap:sec:measure-theory} includes a primier on measure and probability theory for readers who are less familiar with some concepts in this subsection. }
We further write $\cH_{-k}(t) = \cH_{[K] \setminus \{k\}}(t)$ for any $k \in [K]$.
\somesh{Describe measure theory in Appendix}

\begin{defn}\citep{Eichler2017}
 For a $K$-dimensional MPP, event type $k$ is \emph{Granger non-causal} for event type $k'$ if $\lambda^*_{k'}(t)$ is $\cH_{-k}(t)$-measurable for all $t$.
\end{defn}
The above definition amounts to saying that a type $k$ is Granger non-causal for another type $k'$ if, given the history of events other than type $k$, historical events of type $k$ do not further contribute to future $\lambda^*_{k'}(t)$ at any time. Otherwise, type $k$ is said to be \emph{Granger causal} for type $k$.

Uncovering Granger causality from event sequences generally is a very challenging task, as the underlying MPP may have rather complex CIFs with abundant event interaction and non-excitative effect.
As a result, existing work tends to restrict consideration to certain classes of parametric MPPs, such as Hawkes processes \citep{Eichler2017,xu2016learning,achab2017uncovering}.
Specifically, for multivariate Hawkes process, it is straightforward from~\eqref{eq:hawkes-cif} that a type $k$ is Granger non-causal for another type $k'$ if and only if the corresponding kernel function $\phi_{k',k}(\cdot) = 0$.
\deleted{However, Hawkes process assumes that \emph{past events can only independently and additively excite (and not inhibit) the occurrence of future events}; this strong parametric assumption can be easily violated on real-world data.
For example, maintenance events should reduce the chances of a system breaking down, and a patient who takes multiple medicines at the same time is more likely to experience unexpected adverse events.}

\mysubsection{Attribution Methods}
\label{subsec:attribution-methods}

We view an \emph{attribution method} for black-box functions as another ``black box'', which takes in a function, an input, and a baseline, and outputs a set of meaningful attribution scores, one per feature.
The following is a formal definition of attribution method.

\begin{defn}[\textbf{Attribution Method}]
 Suppose $\x\in \cX \subseteq \bbR^d$ is a $d$-dimensional input and $\ulx \in \cX$ a suitable baseline.
 Let $\cF_d$ be a class of functions from $\cX$ to $\bbR$.
 A functional $\Attr: \cF_d \times \cX \times \cX \mapsto \bbR^d$ is called an \emph{attribution method} for $\cF_d$ if $\Attr_i(f, \x, \ulx)$ measures the contribution of $x_i$ to the prediction $f(\x)$ relative to $\ulx$ for any $f \in \cF_d$, $\x, \ulx \in \cX$, and $i \in [d]$.
 \label{def:attribution-method}
\end{defn}

Since it is very challenging (and often subjective) to compare different attribution methods, \citet{Sundararajan2017} argue that attribution methods should ideally satisfy a number of axioms (i.e., desirable properties), which we re-state in Definition~\ref{def:attribution-axioms}.
\begin{defn}
An attribution method $\Attr$ is said to satisfy the axiom of:
\begin{enumerate}
 \item \emph{Linearity}, if for any $f, g \in \cF_d$, $\x, \ulx \in \cX$, and $c \in \bbR$,
 \begin{equation}
 \begin{aligned}
 & \Attr(f, \x, \ulx) + \Attr(g, \x, \ulx) = \Attr(f + g, \x, \ulx), \\
 & \Attr(c f, \x, \ulx) = c \cdot \Attr(f, \x, \ulx).
 \end{aligned}
 \tag{A1} \label{P:linearity}
 \end{equation}

 \item \emph{Completeness/Efficiency}, if for any $f \in \cF_d$ and $\x, \ulx \in \cX$,
 \begin{equation}
 f(\x) - f(\ulx) = \sum_{i = 1}^d \Attr_i(f, \x, \ulx). \tag{A2} \label{P:completeness}
 \end{equation}


 \item \emph{Null player}, if for any $f \in \cF_d$ such that $f$ does not mathematically depend on a dimension $i$,
 \begin{equation}
 \Attr_i(f, \x, \ulx) = 0, \tag{A3}
 \label{P:null-player}
 \end{equation}
 for all $\x, \ulx \in \cX$.

 \item \emph{Implementation invariance}, if for any $\x, \ulx \in \cX$, and any $f, g \in \cF_d$ such that $f(\x') = g(\x')$ for all $\x' \in \cX$,
 \begin{equation}
 \Attr(f, \x, \ulx) = \Attr(g, \x, \ulx). \tag{A4}
 \label{P:impl-invariance}
 \end{equation}

\end{enumerate}

\label{def:attribution-axioms}
\end{defn}

Besides these four axioms, we also identify two other useful properties of attribution methods, which are less explicitly mentioned in the literature.
We state these two properties in Definition~\ref{def:attribution-extra-properties}.

\begin{defn}
An attribution method $\Attr$ is said to satisfy
\begin{enumerate}
 \item \emph{Fidelity-to-control}, if for any $f \in \cF_d$, $\x, \ulx \in \cX$, and $i \in [d]$,
 \begin{equation}
 x_i = \underline{x}_i \enskip \Rightarrow \quad \Attr_i(f, \x, \ulx) = 0. \tag{P1} \label{P:fidelity-to-control}
 \end{equation}

 \item \emph{Batchability}, if for any $f \in \cF_{d}$ and any $n$ input/baseline pairs ${\{(\x_i, \ulx_i)\}}_{i \in [n]}$, there exists a function $F: \cX^n \mapsto \bbR$ such that
 \begin{equation}
 \Attr(F, \X, \ulX) = {[\Attr(f, \x_1, \ulx_1), \ldots, \Attr(f, \x_n, \ulx_n)]}, \tag{P2} \label{P:batchability}
 \end{equation}
 where $\X \defas [\x_1, \ldots, \x_n]$ and $\ulX \defas [\ulx_1, \ldots, \ulx_n]$.
\end{enumerate}
\label{def:attribution-extra-properties}
\end{defn}


Many popular attribution methods satisfy most of these six properties, as we show in the Proposition~\ref{prop:IG-DeepLIFT-properties} and~\ref{prop:shapley-properties}.

\begin{prop}
 Integrated Gradients \citep{Sundararajan2017} satisfies all four axioms (\ref{P:linearity}--\ref{P:impl-invariance}) and two properties (\ref{P:fidelity-to-control}--\ref{P:batchability}), and DeepLIFT \citep{shrikumar2017learning} satisfies all but the implementation invariance (\ref{P:impl-invariance}).
 In particular, a choice of $F$ for both methods satisfying batchability~\eqref{P:batchability} is $F(\X) \defas \sum_{i = 1}^n f(\x_i)$.
 \label{prop:IG-DeepLIFT-properties}
\end{prop}

\begin{prop}
 For any $U\subseteq [d]$, let $\bar{U} \defas [d] \setminus U$ and define $\x_{U} \sqcup \ulx_{\bar{U}}$ to be the spliced data point between $\x$ and $\ulx$ such that for any $i \in [d]$
 \begin{equation}
 [\x_{U} \sqcup \ulx_{\bar{U}}]_i \defas \begin{cases}
 \x_i & i \in U, \\
 \ulx_i & i \in \bar{U}.
 \end{cases}
 \label{eq:spliced-data}
 \end{equation}
 Then Shapley values \citep{shapley1953value} with a value function $v_{f, \x, \ulx}(U) \defas f(\x_{U} \sqcup \ulx_{\bar{U}})$ satisfies all four axioms (\ref{P:linearity}--\ref{P:impl-invariance}) and the fidelity-to-control (\ref{P:fidelity-to-control}).
 \label{prop:shapley-properties}

\end{prop}

\begin{proof}
 We include proofs for both propositions in Appendix~\ref{ap-sec:techinical}.
\end{proof}

\mysection{Proposed: \method}
\label{sec:method}
In this section, we formally present \method, a novel framework for learning Granger causality from multi-type event sequences.
Our framework consists of two steps: first, it trains a neural point process (NPP) to fit the underlying event sequence data; then it inspects the predictions of the trained NPP to compute a Granger causality statistic with some ``well-behaved'' attribution method $\Attr(\cdot)$, which we assume satisfies the following properties: linearity~\eqref{P:linearity}, completeness~\eqref{P:completeness}, null player~\eqref{P:null-player}, fidelity-to-control~\eqref{P:fidelity-to-control}, and batchability~\eqref{P:batchability}.

In what follows, we first describe the architecture of the used NPP in Section~\ref{subsec:npp-architecture}.
Then we elaborate the intuition and the definition of our Granger causality statistic in Section~\ref{subsec:granger-causality-statistic}.
Section~\ref{subsec:computing-causality-statistic} explains the computational challenges and presents a highly efficient algorithm for computing such statistic.
We conclude this section by discussing the choice of attribution methods for \method in Section~\ref{subsec:choice-of-attribution-method}.

\somesh{Need a picture discribing the overall architecture of \method.}

\mysubsection{A Semi-Parametric Neural Point Process}
\label{subsec:npp-architecture}

The design of our NPP follows the general encoder-decoder architecture of existing NPPs (Section~\ref{subsec:mpp}), but we innovate the decoder part to enable both modeling flexibility and computational feasibility.

\myparagraph{Encoder.}

First, we convert each event $i$ into an embedding vector $\v_i$ that summarizes both the temporal and the type information for that event, as follows:
\begin{equation}
  \v_i = [\vartheta(t_i - t_{i - 1}) ; \V^T \z_i],
\end{equation}
where $\vartheta(\cdot)$ is a pre-specified function that transforms the elapsed time into one or more temporal features (simply chosen to be identity function in our experiments), $\V$ is the embedding matrix for event types, and recall that $\z_i$ is the one-hot encoding of the even type $k_i$.

We then obtain the embedding of a history from event embedding sequences by
\begin{equation}
  \h_i = \SeqEnc(\v_1, \v_2, \ldots, \v_i),
  \label{eq:seq-encoder}
\end{equation}
where $\SeqEnc(\cdot)$ is a sequence encoder\replaced{ and chosen to be a Gated Recurrent Unit (GRU) \protect\citep{Cho2014} in our experiments.}{. It is chosen to be a Gated Recurrent Unit (GRU) \mbox{\citep{Cho2014}} in our experiments, although alternative architectures, such as LSTMs, RNNs with attention, and transformers \mbox{\citep{Vaswani2017}}, are also applicable.}

\myparagraph{Decoder.}

An ideal decoder should fullfill the following two desiderata: (a) it should be \emph{flexible enough} to produce from $\h_i$ a wide variety of $\lambda^*_k(t)$ with complex time-varying patterns; and (b) it should also be \emph{computationally manageable}, particularly in terms of computing the cumulative intensity $\int_{t_i}^{t_{i + 1}} \lambda^*_k(t') dt'$, a key term in the log-likelihood-based training given in~\eqref{eq:nll-mpp} \emph{and} the definition of our Granger causality statistic in the subsequent subsections.

We propose a novel semi-parametric decoder that enjoys both the flexible modeling of CIF and computational feasibility.
Specifically, for each $i \in [n]$, we define the CIF $\lambda^*_k(t)$ on $(t_i, t_{i+1}]$ to be a weighted sum of a set of basis functions, as follows:
\begin{equation}
  \lambda^*_k(t) = \sum_{r = 1}^R \alpha_{k, r}(\h_i) \psi_r(t - t_i ),
  \label{eq:CIF}
\end{equation}
where $\{\psi_r(\cdot)\}_{r = 1}^R$ is a set of pre-specified positive-valued basis functions, and $\balpha: \bbR^{N_h} \mapsto \bbR_+^{K \times R}$ is a feedforward neural network that computes $R$ positive weights for each of the $K$ event types.
In this way, by choosing $\{\psi_r(\cdot)\}_{r = 1}^R$ to be a rich-enough function family, the CIFs defined in~\eqref{eq:CIF} are able to express a wide variety of time-varying patterns. More importantly, since the parts relevant to neural networks---$\balpha(\cdot)$ and $\SeqEnc(\cdot)$---are separated from the basis functions, we can evaluate the integral $\int_{t_i}^{t_{i + 1}} \lambda^*_k(t') dt'$ analytically, as follows:
\begin{equation}
  \int_{t_i}^{t_{i + 1}} \lambda^*_k(t') dt' = \sum_{r = 1}^R \alpha_{k, r}(\h_i) \Psi_r(t_{i + 1} - t_i),
  \label{eq:cumulative-intensity-close-form}
\end{equation}
where $\Psi_r(\Delta t) \defas \int_{0}^{\Delta t} \psi_r(t)\,dt$ is generally available for many parametric basis functions.

\replaced{We choose the basis functions $\{\psi_r(\cdot)\}_{r = 1}^R$ to be the densities of a Gaussian family with increasing means and variances. This design of basis functions reflects a reasonable inductive bias that the CIFs should vary more smoothly as the time increases. The details are given in Appendix~\ref{ap:subsec:dyadic-gaussian-basis}.
}{
Inspired by the dyadic interval bases used by \citet{Bao2017}, we choose the basis functions $\{\psi_r(\cdot)\}_{r = 1}^R$ to be the densities for a Gaussian family $\{\cN(\mu_r, \sigma_r^2) \}_{r = 1}^R$, whose means are given by
\begin{equation}
  \mu_r = \begin{cases}
    0, & r = 1, \\
    L / 2^{R - r},  & r = 2, \ldots, R,
  \end{cases}
\end{equation}
and the standard deviations by $\sigma_r = \max(\mu_r / 3, \mu_2 / 3)$ for $r \in [R]$.
This design of basis functions reflects a reasonable inductive bias that the CIFs should vary more smoothly as the time increases. As shown in Figure~\ref{fig:dyadic-normal-example} for an example of $L=100$ and $R=5$, the first a few bases, due to their small means and variances, capture the short-term effects, whereas the last several characterize the mid/long-term effects.
}



\mysubsection{From Event Contributions to a Granger Causality Statistic}
\label{subsec:granger-causality-statistic}

Now that we have trained a flexible NPP that can successively update the history embedding after each event $i$ occurrence and then predict the future CIFs $\lambda^*_k(t)$ after $t_i$ until $t_{i + 1}$;
we would like to ask: \emph{can we quantify the contribution of each past event to each prediction?}
Since in our case $\lambda^*_k(t)$'s are instantiated by two potentially highly nonlinear neural networks (i.e., $\SeqEnc(\cdot)$ and $\balpha(\cdot)$), it is not as straightforward to obtain the past event's contribution to current event occurrence as in the case of some parametric MPPs (e.g., Hawkes processes).

A natural idea for the aforementioned question would be applying some attribution method to $\lambda^*_k(t)$'s. To do so, however, there are two challenges.
First, the predictions in our case are time-varying functions rather than static quantities (e.g., the probability of a class, as commonly seen in existing applications of attribution methods); thus it is unclear which target should be attributed.
Second, as the input to $\lambda^*_k(t)$'s are multi-type event sequences with asynchronous timestamps, it is also unclear which baseline should be used.

We tackle the first challenge by setting the \emph{cumulative intensity} $\int_{t_i}^{t_{i + 1}} \lambda^*_k(t')\,dt'$ to be the attribution target.
This is not only because the cumulative intensity reflects the overall effect of $\lambda^*_k(t')$ on $(t_i, t_{i + 1}]$, but also because it has a clear meaning in the context of point processes: it is the rate of the Poisson distribution that the number of events of type $k$ on $(t_i, t_{i + 1}]$ satisfies.
More importantly, since the cumulative intensity has a closed form as in~\eqref{eq:cumulative-intensity-close-form}, its gradients with respect to its input can be computed both precisely and efficiently.
Note that by adopting this target, the input now includes not only $\{(t_i, \z_j)\}_{j \leq i}$ but also $t_{i + 1}$; thus we define $\x_i \defas [t_1, \z_1, \ldots, t_{i}, \z_{i}, t_{i + 1}]$ and write the target $\int_{t_i}^{t_{i + 1}} \lambda^*_k(t')\,dt'$ as $f_k(\x_i)$.

As for the second challenge, we choose the baseline of an input $\x_i$  to be $\ulx_i \defas [t_1, \0, \ldots, t_{i - 1}, \0, t_{i + 1} ]$; that is, the one-hot encodings of all observed event types are replaced with zero vectors.
Since $\x_i$ and $\ulx_i$ only differ in the dimensions corresponding to the event types, i.e., $\lbrace z_{j, k_j} \rbrace_{j \leq i}$, by the fidelity-to-control~\eqref{P:fidelity-to-control}, then only these dimensions will have non-zero attributions.
With completeness~\eqref{P:completeness}, it further implies that for every type $k$
\begin{equation}
  f_k(\x_i) - f_k(\ulx_i) = \sum_{j = 1}^{i} \Attr_j(f_k, \x_i, \ulx_i),
  \label{eq:type-IG-completeness}
\end{equation}
where $\Attr_j(f_k, \x_i, \ulx_i )$ is the attribution to $z_{j, k_j}$.
Thus, the term $\Attr_j(f_k, \x_i, \ulx_i)$ can be viewed as the \emph{event contribution} of the $j$-th event to the cumulative intensity prediction  $f_k(\x_i)$ relative to the baseline $f_k(\ulx_i)$.
Besides, event \emph{timestamps} are identical in $\x_i$ and $\ulx_i$, thus this contribution comes only from the event \emph{type} and denotes how type $k_j$ contributes to the prediction of type $k$ for a specific event history $\x_i$.

\myparagraph{A Granger Causality Statistic.}

We have established $\Attr_j(f_k, \x_i, \ulx_i)$'s as the past events' contribution to the cumulative intensity $f_k(\x_i)$ on interval $(t_i, t_{i + 1}]$.
A further question is: \emph{can we infer from these event contributions for individual predictions the population-level Granger causality among event types?}

To answer this question, we define a novel statistic indicating the Granger causality for type $k'$ to type $k$ as follows:
\begin{equation}
  Y_{k, k'} \defas \frac{\sum_{s = 1}^S \sum_{i = 1}^{n_s} \sum_{j=1}^{i}  \bbI(k^s_j = k') \Attr_j(f_k, \x^s_i, \ulx^s_i)}{\sum_{s = 1}^S \sum_{j = 1}^{n_s} \bbI(k^s_j = k')}.
\label{eq:causality-statistic}
\end{equation}
Here the numerator aggregates the event contributions for all event occurrences over the whole dataset, and denominator accounts for the fact that some event types may occur far more frequently than other types, which can lead to unreasonally large scores if used without such normalization.
Note that an event contribution $\Attr_j(f_k, \x^s_i, \ulx^s_i)$ may be negative when the event $j$ exerts an inhibitive effect; thus $Y_{k, k'}$ can also be negative and characterize the Granger causality from type $k'$ to type $k$ even when the influence is inhibitive.

\inlinecomment{Prasad: What's the connection between this statistics and the definition of 1 of Granger causality?}

\myparagraph{Attribution Regularization.}

One caveat in~\eqref{eq:type-IG-completeness} and \eqref{eq:causality-statistic} is that our chosen baselines $\ulx_i$ have never appeared in the training procedure, thus the value of $f(\ulx_i)$ may be meaningless or even misleading.
Ideally, we would like $f_k(\ulx_i)$ to be the cumulative intensity of type $k$ given that history prior to $t_i$ consists of ``null'' events at $t_1, t_2, \ldots, t_i$.
Thus a natural prior reflecting this idea is to make $f_k(\ulx_i)$ nearly zero for any handcrafted baseline $\ulx_i$.
Such an ``invariance'' property on $f$ can be achieved by adding an auxiliary $l_1$ regularization for each $\x_i$ in the NLL given in~\eqref{eq:nll-mpp}, leading to a training objective
\begin{equation}
  \sum_{s=1}^S \sum_{i = 1}^{n_s} \Big\{ \underbrace{- \log \lambda^*_{k^s_i}(t^s_i) + \sum_{k = 1}^K  f_k(\x^s_{i - 1})}_{\text{negative log-likelihood}} + \underbrace{\sum_{k = 1}^K  \eta f_k(\ulx^s_{i - 1})}_{\text{regularization}} \Big\},
  \label{eq:objective-erpp}
\end{equation}
where $\eta$ is a hyperparameter.

\mysubsection{Computing the Granger Causality Statistic}
\label{subsec:computing-causality-statistic}

While~\eqref{eq:causality-statistic} defines $Y_{k, k'}$'s analytically, it is rather challenging to compute them.
This is because a naive implementation would require applying $\Attr(\cdot)$ at each event occurrence, which is computationally prohibitive for a dataset of millions of events.
Note that the normalization in~\eqref{eq:causality-statistic} can be easily calculated; so if we write $\tdY^s_{k, k'} \defas \sum_{i = 1}^{n_s} \sum_{j=1}^{i}  \bbI(k^s_j = k') \Attr_j(f_k, \x^s_i, \ulx^s_i)$, the problem is reduced to how to efficiently compute $\sum_{s = 1}^S \tdY^s_{k, k'}$.

We propose an efficient algorithm to compute $\sum_{s = 1}^S \tdY^s_{k, k'}$, which is summarized in Algorithm~\ref{alg:batching-scheme}.
At the core of our algorithm are two levels of batching:
(a) \emph{intra-sequence batching}, which allow the computation of $\tdY^s_{k, k'}$ with only one call of $\Attr(\cdot)$;
and (b) \emph{inter-sequence batching}, which enables  batch computation of $\lbrace Y^s_{k,k'} \rbrace_{s \in \cB}$ for a mini-batch of event sequences indexed by $\cB$. We explain the details of these two levels of batching as follows.

\begin{algorithm2e}[!tbp]
  \caption{Computation of the Granger causality statistic.}
  \label{alg:batching-scheme}
  \KwIn{Event sequences $ \lbrace \lbrace (t^s_i, k^s_i) \rbrace_{i \in [n_s]} \rbrace_{s \in [S]}$, an attribution method $\Attr(\cdot)$, and a trained NPP}
  \KwOut{Granger causality statistic $\Y$.}

  Initialize $ \mathbf{\tdY} = \0$, $\cI = [S]$ \;
  \While {$|\cI| > 0$}{
    Sample a batch of sequence indices $\cB \subset \cI$ \;
    \For {$k = 1, \ldots, K$}{
      Compute $\C = \Attr(\sum_{s \in \cB} \sum_{i = 1}^{n_s} F^s_{k, i}, \X, \ulX)$\; \label{alg:line:call-attribution}
      \For {$k' = 1, \ldots, K$}{
        $\tdY_{k, k'} \pluseq \sum_{s \in \cB} \sum_{j = 1}^{n_s} \bbI (k^s_j = k') C^s_j$
      }
    }
    $\cI \leftarrow \cI \setminus \cB$\;
  }
  Compute $Y_{k, k'} = \tdY_{k, k'} / \sum_{s = 1}^S \sum_{j = 1}^n \bbI(k_j^s = k')$, $\forall k, k' \in [K]$.

\end{algorithm2e}

\myparagraph{Intra-Sequence Batching.}
As this part only deals with a particular event sequence, to simplify the notation, we omit the sequence index $s$ for now.
Note that $\x_1 \prec  \x_2 \prec  \cdots \prec \x_n$ and due to the recurrent nature of $f$, all $f(\x_i)$ for $i \in [n]$ can be computed in a single forward pass with the shared input $\x_n$.
Denote by $F = \lbrace F_{k, i}(\cdot) \rbrace_{k \in [K], i \in [n]}$ a matrix-valued function such that $F_{k, i}(\x_n) = f_k(\x_i)$ for any $k \in [K], i \in [n]$.

The equivalence between $f$ and $F$ means that,
\begin{equation*}
  \Attr_j(f_k, \x_i, \ulx_i) = \Attr_j(F_{k, i}, \x_n, \ulx_n),
\end{equation*}
which further implies that we can rewrite $\tdY^s_{k, k'}$ as a weighted sum of attribution scores for the same input $\x_n$ and baseline $\ulx_n$. Since we are not interested in computing the individual attribution scores but their sum, we can leverage the linearity property~\eqref{P:linearity} to compute the attribution scores directly for the sum, as shown in the following proposition.
\begin{prop}
  For an attribution method $A(\cdot)$ with the linearity~\eqref{P:linearity} and the null player~\eqref{P:null-player}, it holds that
  \begin{equation}
    \tdY^s_{k, k'} = \sum_{j = 1}^{n_s} \bbI (k^s_j = k') \Attr_j \left( \sum_{i= 1}^{n_s} F_{i, k_i} , \x^s_n, \ulx^s_n \right).
  \end{equation}
  \label{prop:intra-seq-batching}
  \reducemargin
\end{prop}
\begin{proof}
  The proof is in Appendix~\ref{ap:subsec:proof-intra-seq-batching}.
\end{proof}

\myparagraph{Inter-Sequence Batching.}

We now discuss how to efficiently compute $\lbrace Y^s_{k,k'} \rbrace_{s \in \cB}$ for a mini-batch of event sequences indexed by $\cB$.
The key idea for a significant computational speed-up here is that if $\Attr(\cdot)$ satisfies batchability~\eqref{P:batchability},  we can then batch the computation of different sequences with a single call of $\Attr(\cdot)$.

To simplify the discussion, we assume without loss of generality that $\cB = \lbrace 1, \ldots, |\cB| \rbrace$ and $n_s \equiv n$ for all $s \in \cB$.
Let $\X = [\x_s]_{s \in \cB}$ and analogously the corresponding baselines $\ulX$.
We further override our previous notation and denote by $F = \lbrace F^s_{k, i}(\cdot) \rbrace_{s \in [S], k \in [K], i \in [n]}$ a new tensor-valued function such as that $F^{s}_{k, i}(\X) = f_k(\x^s_i)$.
Then with Proposition~\ref{prop:IG-DeepLIFT-properties}, we have that
\begin{equation*}
  \Attr(\sum_{s \in \cB} \sum_{i = 1}^{n_s} F^s_{k, i}, \X, \ulX) = \left[ \Attr_j ( \sum_{i= 1}^{n_s} F_{i, k_i} , \x^s_n, \ulx^s_n ) \right]_{\substack{s \in \cB\\ j \in [n]}}.
  \label{eq:event-contribution-batch}
\end{equation*}

\myparagraph{Time Complexity Analysis.}

With our two-level batching scheme, Algorithm~\ref{alg:batching-scheme} only requires $O(S K / B)$ invocations of $\Attr(\cdot)$, a significant reduction from the $O(S N K)$ invocations required by a naive implementation that directly calculates $Y_{k,k'}$'s, where $N$ is the average sequence length.
Since modern computation hardware (such as GPUs) enables calling $\Attr(\cdot)$ with a batch of inputs being almost as fast as calling it with a single input, our algorithm can achieve up to \emph{three orders-of-magnitude speedup} over a naive implementation on datasets with relatively large $N$ and $B$. (See Section~\ref{subsubsec:scalability} for empirical evaluations.)

\mysubsection{Choice of Attribution Methods}
\label{subsec:choice-of-attribution-method}

In our experiments, we choose the attribution method $\Attr(\cdot)$ to be the Integrated Gradients, which is defined as follow:
\begin{equation}
  \IG(f, \x, \ulx) \defas (\x - \ulx) \odot \int_{0}^1 \left. \frac{\partial f(\tilde{\x})}{\partial \tilde{\x}} \right\vert_{\tilde{\x}=\ulx + \alpha (\x - \ulx)} d\alpha,
  \label{eq:ig-defn}
\end{equation}
where $\odot$ is the Hadamard product.
Nevertheless, \method does not depend on a specific attribution method but a set of properties that we have stated upfront;
this means that any other attribution methods that satisfy these properties (e.g., DeepLIFT) should be applicable to \method.
Also note that batchability~\eqref{P:batchability} is only used in the inter-sequence batching for speeding up the computation; thus, if efficiency is less of a concern, or the computation of attributions for different inputs can be accelerated in alternative ways,\footnote{In fact, for almost all attribution methods, the attribution for different inputs is embarrassingly parallelizable.} attribution methods that only violate batchability, such as Shapley values, should also be applicable.

\mysection{Experiments}
\label{sec:experiments}






In this section, we present the experiments that are designed to evaluate \method and answer the following three questions:
\begin{itemize}
  \item \textbf{Goodness-of-Fit}: How good is \method at fitting multi-type event sequences?
  \item \textbf{Causality Discovery}: How accurate is \method at discovering Granger causality between event types?
  \item \textbf{Scalability}: How scalable is \method?
\end{itemize}

The experimental results on five datasets show that \method (a) outperforms state-of-the-art methods in both fitting and discovering Granger causality from event sequences of diverse event interdependency, (b) can identify Granger causality on real-world datasets that agrees with human intuition, and (c) can compute the Granger causality statistic three orders-of-magnitude faster due to our optimization.

\mysubsection{Experimental Setup}

\paragraph{Datasets.}

We designed three synthetic datasets to reflect various types of event interactions and temporal effects.
\begin{itemize}
 \item \textbf{\Excitation}: This dataset was generated by a multivariate Hawkes process, whose CIFs are defined in~\eqref{eq:hawkes-cif}. The exponential decay kernels were used, and a weighted ground-truth causality matrix was constructed with the $\ell_1$ norms of the kernel functions $\phi_{k, k'}(\cdot)$.

 \item \textbf{\Inhibition}: This dataset was generated by a multivariate self-correcting process \citep{isham1979self}, whose CIFs are of the form: $\lambda^*_k(t) = \exp(\alpha_k t + \sum_{i: t_i < t} w_{k, k_i})$, where $a_k > 0$ and $w_{k, k'} \leq 0$.
 A weighted ground-truth causality matrix was formed with the pairwise weights $w_{k, k'}$.

 \item \textbf{\Synergy}: Generated by a proximal graphical event model (PGEM) \citep{Bhattacharjya2018}, this dataset contains synergistic effects between a pair of event types to a third event type. A binary ground-truth causality matrix was constructed from the dependency graph of the PGEM.
\end{itemize}

We also included two real-world datasets used in existing literature.
\begin{itemize}
 \item \textbf{\IPTV} \citep{Luo2015}: Each sequence records the history of TV watching behavior of a user, and the event types are the TV program categories. This dataset, however, does not contain ground-truth causality between TV program categories.
 \item \textbf{MemeTracker (\MT)}:\footnote{\url{https://www.memetracker.org/data.html}} Each sequence represents how a phrase or quote appeared on various online websites over time during the period of August 2008 to April 2009, and the event types are the domains of the top websites.
 Like previous studies \citep{achab2017uncovering,Xiao2019}, a weighted ground-truth causality matrix was \emph{approximated} by whether one site contains any URLs linking to another site.

\end{itemize}

The parameter settings for synthetic datasets, the preprocessing steps for real-world datasets, and the dataset statistics are detailed in Appendix~\ref{ap:subsec:datasets}.

\myparagraph{Methods for Comparsison.}

We compared our method to the following baselines:
\begin{itemize}
 \item \textbf{HExp}: Hawkes process with fixed exponential kernels.
 \item \textbf{HSG} and \textbf{NHPC}: Hawkes process with sum of Gaussian kernels \citep{xu2016learning} and nonparametric Hawkes process cumulant matching \citep{achab2017uncovering}. These two methods represent the state-of-the-art parametric and nonparametric methods for learning Granger causality for Hawkes process, respectively.
 \item \textbf{RPPN}: Recurrent point process network \citep{Xiao2019}, to the best of our knowledge, the only NPP that can provide summary statistics for Granger causality, which is enabled by its use of an attention mechanism.
\end{itemize}

The implementation details and hyperparameter configurations for \method and various baselines are provided in Appendix~\ref{ap:subsec:method-config}

\myparagraph{Evaluation Metrics.}

The hold-out negative log-likelihood (\textbf{NLL}) was used for evaluating the goodness-of-fit of each method on various datasets, and the \textbf{Kendall's $\btau$} coefficient and the area under the ROC curve (\textbf{AUC}) were used for evaluating the estimated Granger causality matrix against the ground truth.
Non-binary ground-truth causality matrices were binarized at zero in the evaluation of AUCs.
We performed five-fold cross-validation and report the average results.

\subsection{Detailed Results}

\subsubsection{Goodness-of-fit}

We start by examining the goodness-of-fit of various methods on various datasets, since if a method fails to fit the data, it is unlikely to detect the true Granger causality between event types.
As shown in Figure~\ref{fig:barplot-nll}, \method attains smaller NLLs than all baselines on all datasets, suggesting that \method consistently has a better fit than all baselines. Notably, on all three synthetic datasets, the NLLs of \method nearly match those computed by the ground-truth models.
These results confirm the flexibility of \method in learning the various types of event interactions and temporal effects.

\begin{figure}[!tbp]
  \reducemargin
  \centering
  \subfloat[\Excitation]{\includegraphics[width=0.33\columnwidth]{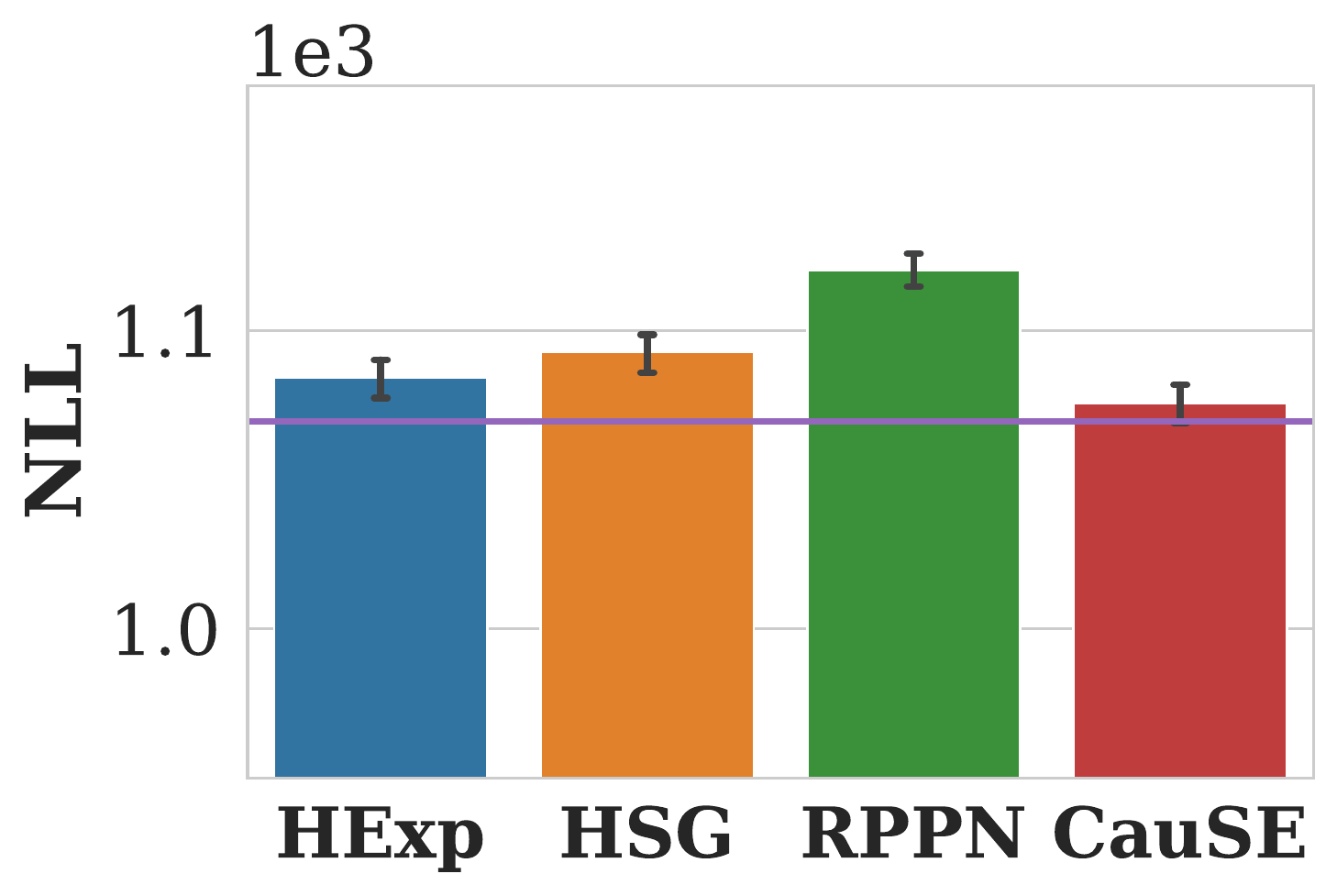}}
  \subfloat[\Inhibition]{\includegraphics[width=0.33\columnwidth]{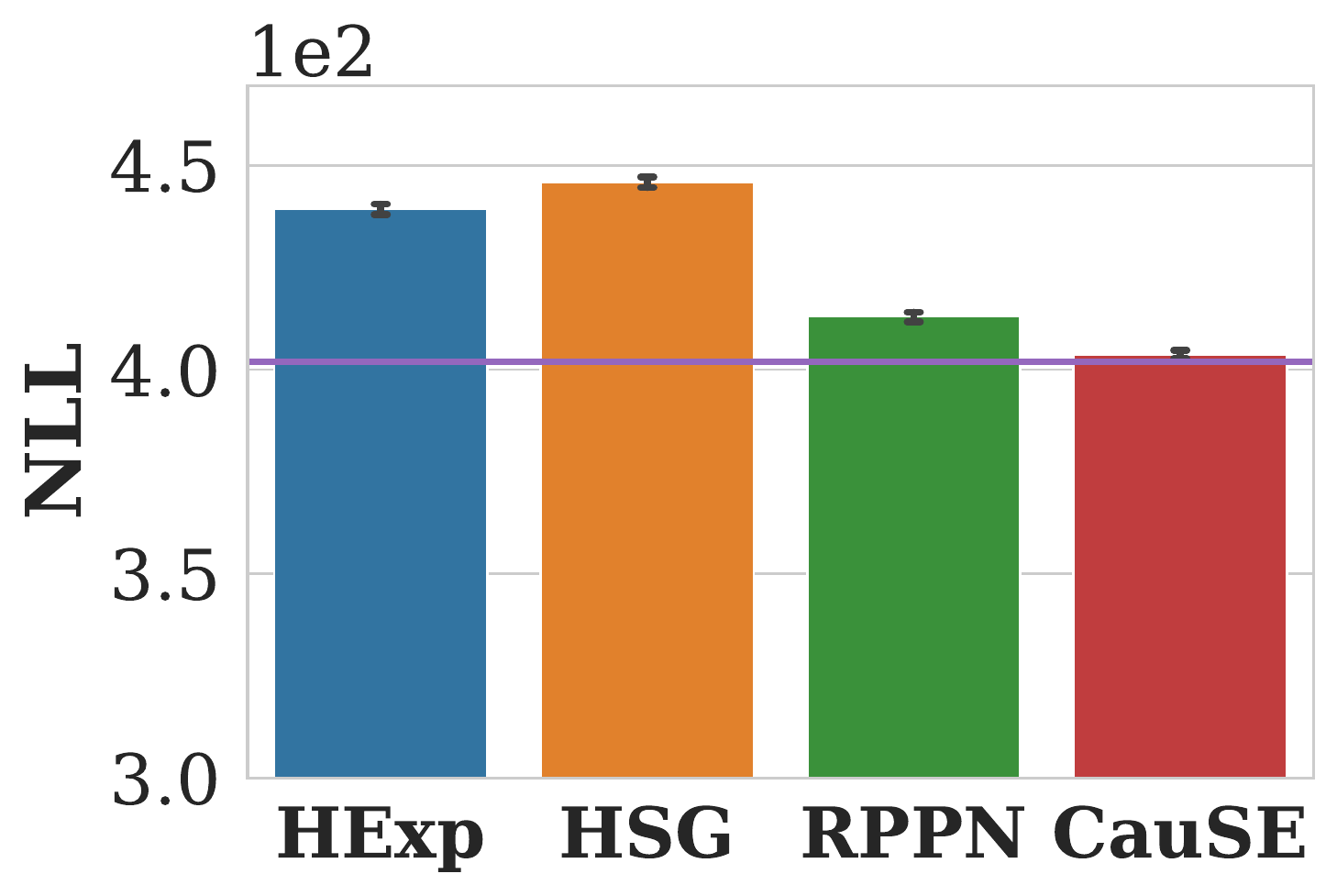}}
  \subfloat[\Synergy]{\includegraphics[width=0.33\columnwidth]{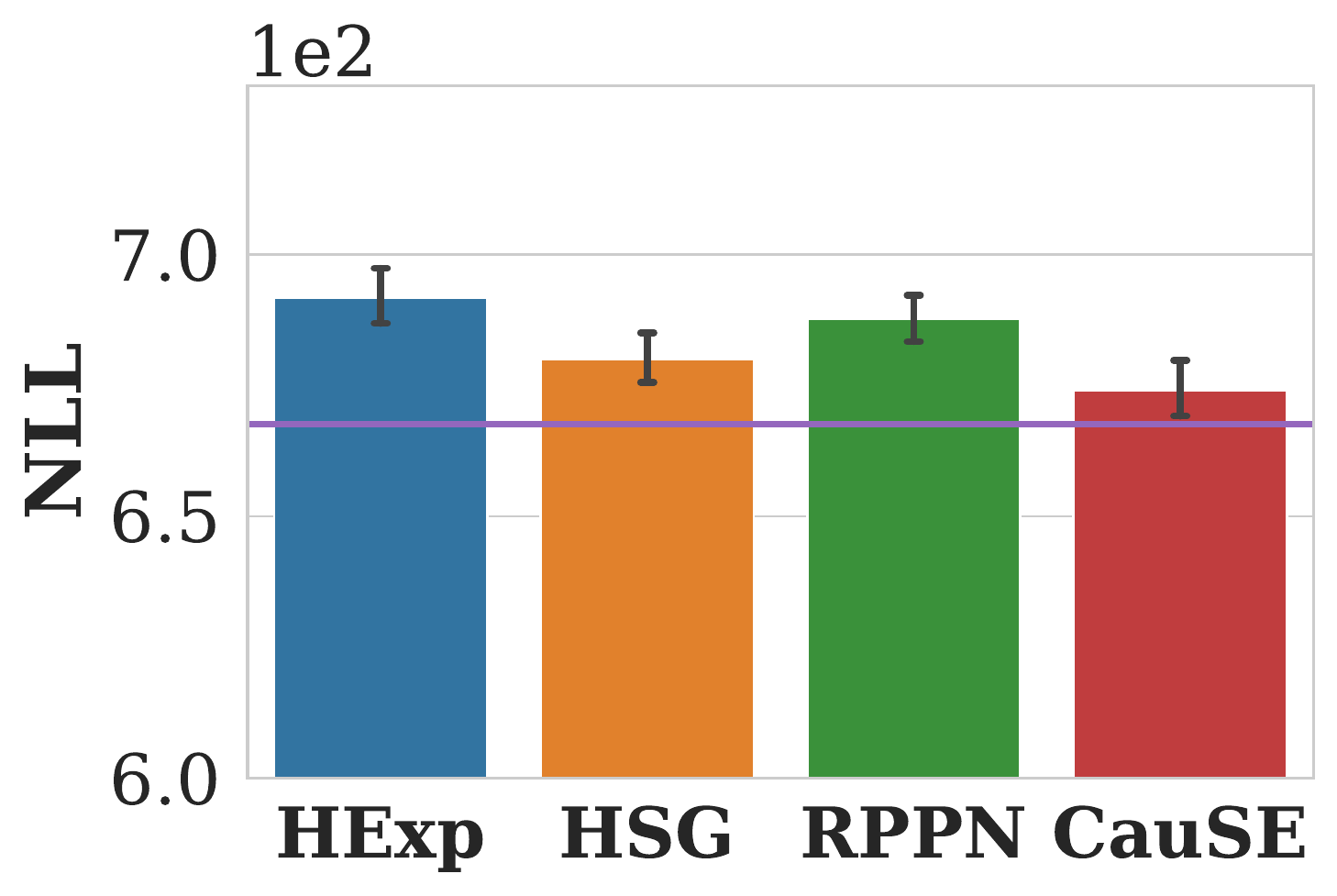}}

  \subfloat[\IPTV]{\includegraphics[width=0.36\columnwidth]{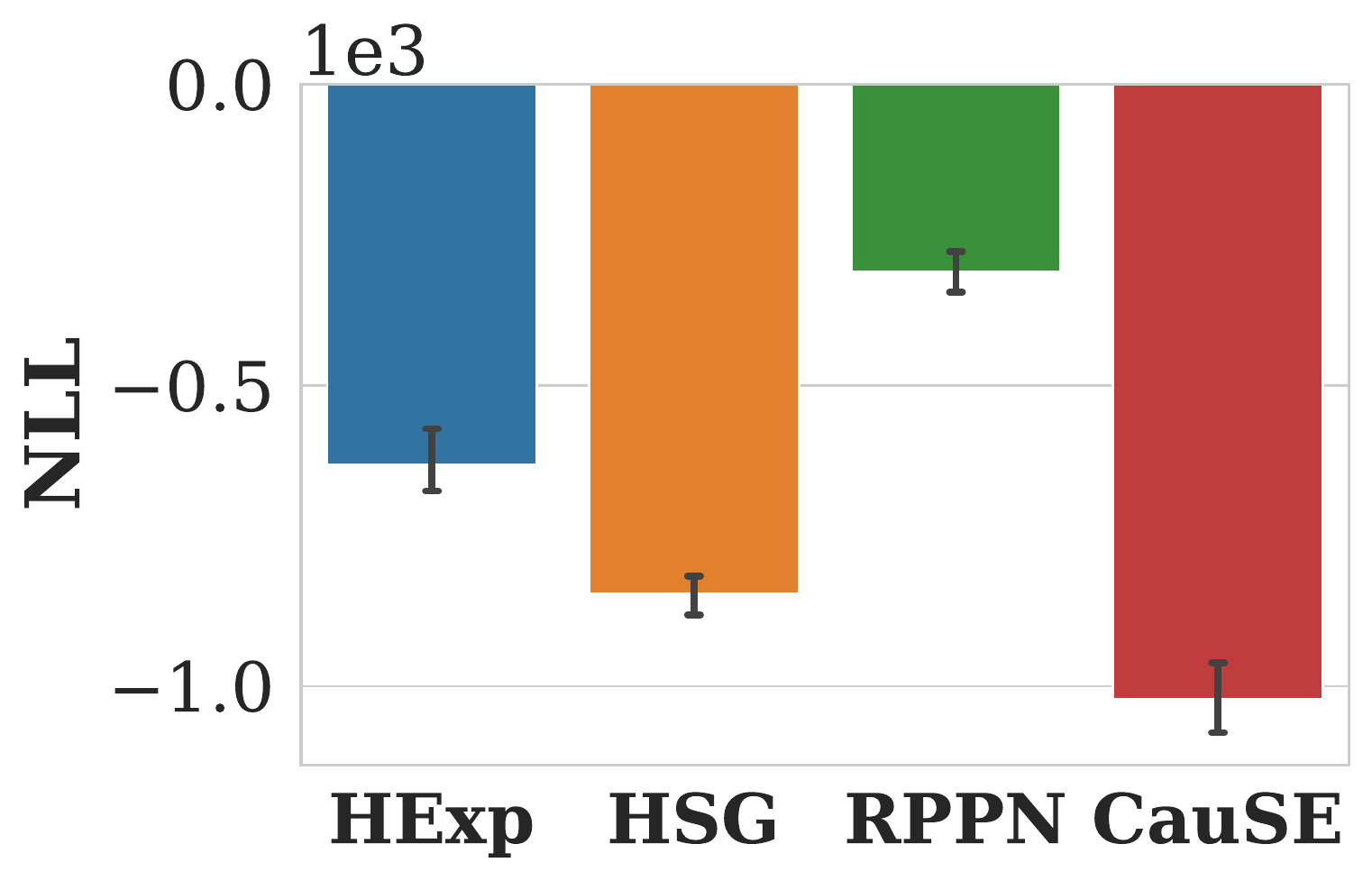}}
  \subfloat[\MT]{\includegraphics[width=0.33\columnwidth]{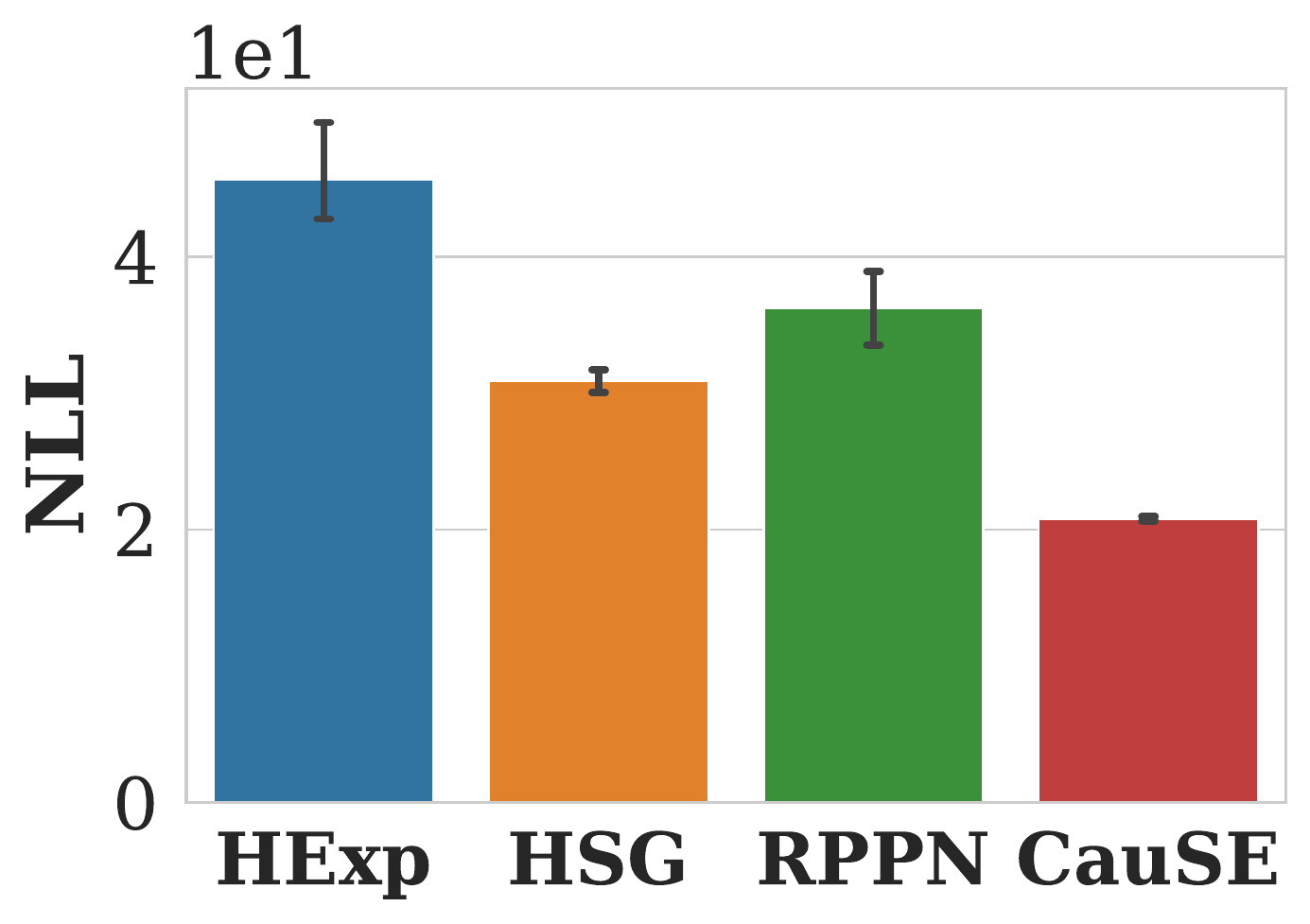}}

  \caption{Hold-out NLLs of various methods, where horizontal lines denote the ground-truth NLLs. \method attains the best NLLs on all datasets.
  }

  \label{fig:barplot-nll}
\end{figure}

\mysubsubsection{Causality Discovery}

We now examined the performance of \method on Granger causality discovery, both quantitatively and qualitatively.

\myparagraph{Quantitative Analsyis.}

\begin{table*}[tbp]

  \centering
  \caption{Results for Granger causality discovery on the four datasets with ground-truth causality available. The best and the second best results on each dataset are emboldened and italicized, respectively.
  }


  \resizebox*{\textwidth}{!}{
  \begin{tabular}{crrrrrrrr}
\toprule
	& \multicolumn{2}{c}{\textbf{\Excitation}} & \multicolumn{2}{c}{\textbf{\Inhibition}}
	& \multicolumn{2}{c}{\textbf{\Synergy}} & \multicolumn{2}{c}{\textbf{\MT}} \\
\cmidrule(lr){2-3}
\cmidrule(lr){4-5}
\cmidrule(lr){6-7}
\cmidrule(lr){8-9}
\textbf{Method}	& AUC & Kendall's $\tau$ & AUC & Kendall's $\tau$ & AUC & Kendall's $\tau$ & AUC & Kendall's $\tau$ \\
\midrule
HExp 		& 0.858$\pm$0.004 & 0.453$\pm$0.005
		 		& \emph{0.546$\pm$0.002} & \emph{0.102$\pm$0.002}
		 		& 0.872$\pm$0.058 & 0.251$\pm$0.039
		 		& 0.404$\pm$0.009 & -0.061$\pm$0.005\\
HSG 		&  \textbf{0.997$\pm$0.001} & \textbf{0.635$\pm$0.002 }
				& 0.490$\pm$0.002 & -0.013$\pm$0.002
				& 0.876$\pm$0.007 & 0.254$\pm$0.039
				& \emph{0.539$\pm$0.008} & \emph{0.024$\pm$0.005} \\
NPHC 		& 0.782$\pm$0.007 & 0.337$\pm$0.010
		 		& 0.400$\pm$0.054 & -0.138$\pm$0.067
		 		& 0.741$\pm$0.129	& 0.163$\pm$0.087
		 		& N/A	& N/A \\
RPPN 		& 0.595$\pm$0.010 & 0.136$\pm$0.012
				& 0.448$\pm$0.003 & -0.066$\pm$0.002
				& \emph{0.891$\pm$0.043} & \emph{0.264$\pm$0.029}
				& 0.492$\pm$0.004 & -0.005$\pm$0.002 \\
\midrule
\method & \emph{0.920$\pm$0.012} & \emph{0.533$\pm$0.013}
				& \textbf{0.921$\pm$0.021} & \textbf{0.532$\pm$0.021}
				& \textbf{0.991$\pm$0.004} & \textbf{0.331$\pm$0.003}
				& \textbf{0.623$\pm$0.012} & \textbf{0.075$\pm$0.007} \\
\bottomrule
\end{tabular}
  }
  \label{tab:granger-causality-results}

  \reducemargin
\end{table*}

Table~\ref{tab:granger-causality-results} shows values of AUC and Kendall's $\tau$ of various methods on the four datasets that have ground-truth causality. The results in the table support the following conclusions.

First, \method performs the best overall and is most robust to various types of event interactions. It not only significantly outperforms all baselines on three of the four datasets (i.e., \Inhibition, \Synergy, and \MT), but also achieves a close-second on \Excitation, in which events were generated by a Hawkes process, and \method is supposed to have a disadvantage relative to Hawkes process-based baselines.

Second, once the underlying data generation process violates the assumptions of Hawkes process and exhibits complex event interactions other than excitation, Hawkes process-based methods tend to perform poorly, as seen from \Inhibition and \Synergy.

Finally, despite both being NPP-based methods, RPPN performs significantly worse than \method on all datasets. We suspect that this underperformance is caused by two issues in RPPN's construction of the Granger causality statistics with the attention weights.
First, RPPN restricts all attention weights to be positive, thus cannot distinguish between excitative and inhibitive effects.
Second, attention mechanism may not correctly attribute the model's prediction to its inputs, as shown in several recent studies \citep{jain2019attention,Serrano2019}.

\myparagraph{Qualitative Analysis.}

Figure~\ref{fig:granger-causality-matrix-IPTV} shows the heat map for the Granger causality matrix of \IPTV dataset estimated by \method.
Almost all diagonal entries have large positive values, indicating that users, on average, exhibit strong tendencies to watch the TV programs of the same category.
Several positive associations between different TV program categories are also observed, such as from military, laws, finance, and education to news, and from kids and music to drama.
These results agree with common sense and are consistent with the findings of an existing study with HSG \citep{xu2016learning}.
Our method also suggests several meaningful negative associations, including ads to drama and education to entertainment; such negative associations, however, can never be detected by models that only consider the excitations between events, such as HSG.

\added{Appendix~\ref{ap:subsec:qualitative-analysis-MT} provides a detailed analysis of the estimated Granger causality matrix for \MT dataset.}

\begin{figure}[!tbp]

  \centering
  \includegraphics[width=0.95\linewidth]{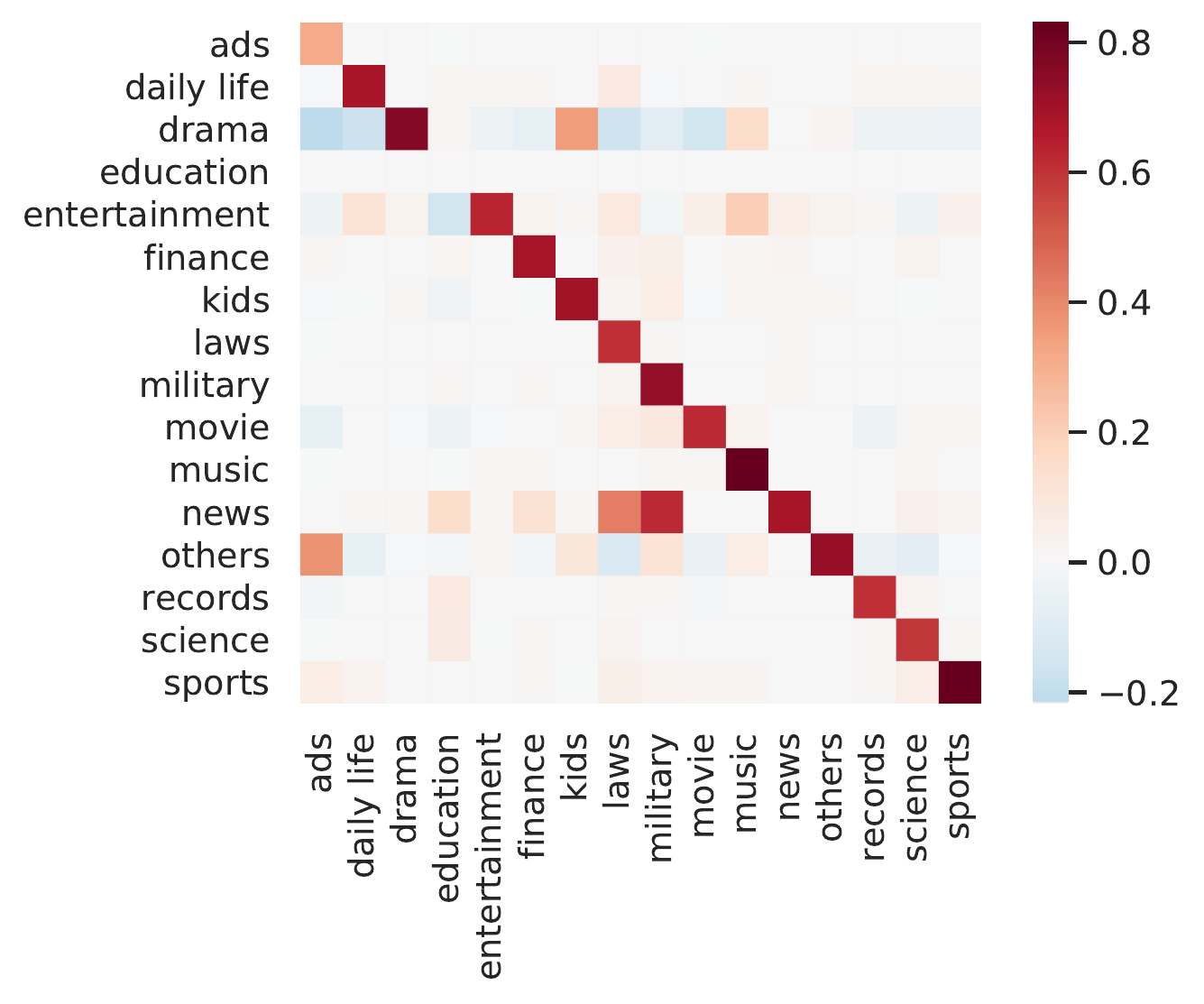}
  \caption{Visualization of the estimated Granger causality statistic matrices on \IPTV \replaced{.}{ and \MT (Figure~\ref{fig:granger-causality-matrix-MT-1} \&~\ref{fig:granger-causality-matrix-MT-2}). Better viewed on screen.}}
  \label{fig:granger-causality-matrix-IPTV}
\end{figure}




\mysubsubsection{Scalability}
\label{subsubsec:scalability}

Finally, we investigate the scalability of \method in computing the Granger causality statistic by Algorithm~\ref{alg:batching-scheme}. Figure~\ref{fig:batching-speedup} shows how much speedup Algorithm~\ref{alg:batching-scheme} achieves over a naive implementation with different sequence lengths and batch sizes.
The results demonstrate that with batch size and average sequence length both being relatively large (i.e., greater or equal to $16$ and $100$, respectively), our algorithm can achieve over \emph{three orders-of-magnitude speedup} relative to a native implementation.
Furthermore, the speedup scales almost linearly with sequence length and batch size when they do not exceed $150$ and $16$, respectively, which is consistent with our analysis in Section~\ref{subsec:computing-causality-statistic}. Beyond this regime, only a sublinear relationship between the speedup and batch size or sequence length is observed, which is because the GPU we tested on was reaching its maximum utilization.

\begin{figure}[!tbp]
  \centering
  \includegraphics[width=0.9\columnwidth]{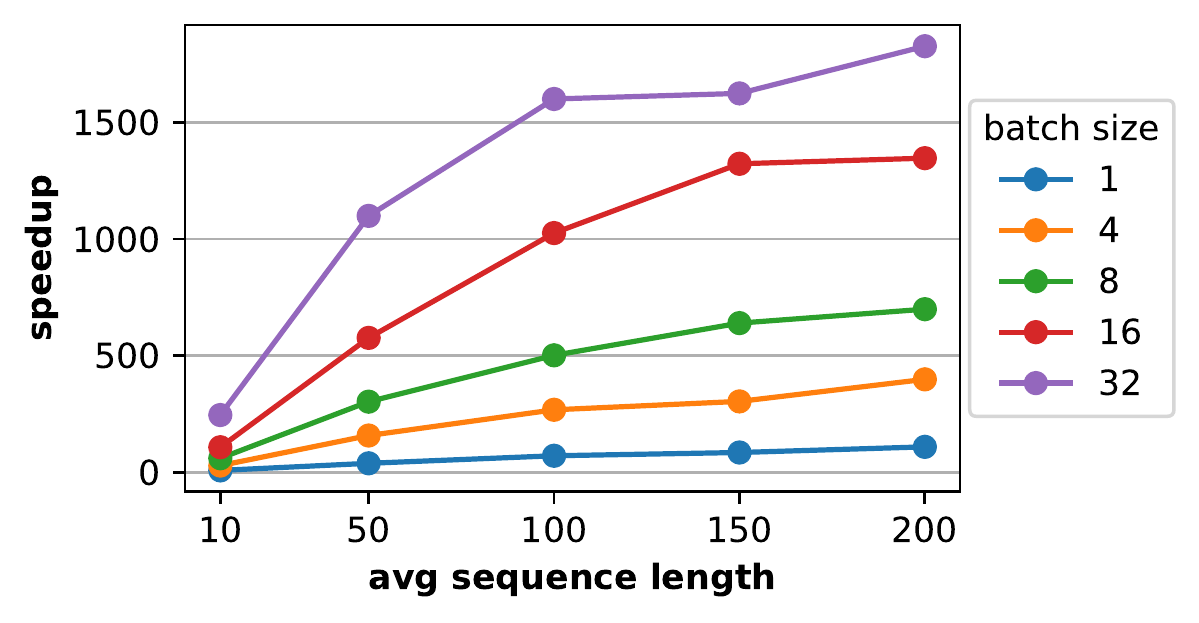}
  \caption{The speedup achieved by Algorithm~\ref{alg:batching-scheme} relative to a naive implementation with different average sequence lengths and batch sizes.}
  \label{fig:batching-speedup}
\end{figure}

\mysection{Conclusion}
\label{sec:conclusion}
We have presented \method, a novel framework for learning Granger causality between event types from multi-type event sequences.
At the core of \method are two steps: first, it trains a flexible NPP model to capture the complex event interdependency, then it computes a novel Granger causality statistic by inspecting the trained model with an axiomatic attribution method.
A two-level batching algorithm is derived to compute the statistic efficiently.
We evaluate \method on both synthetic and real-world datasets abundant with diverse event interactions and show the effectiveness of \method on identifying the Granger causality between event types.


\bibliography{reference}
\bibliographystyle{icml2020}

\appendix
\onecolumn

\section{Additional Related Work}

\paragraph{Event Sequence Modeling}

With the increasing availability of multi-type event sequences, there has been considerable interest in modeling such data for both prediction and inference.
The majority of prior research in this direction has been based on the theory of point processes \citep{Daley2003}, a particular class of stochastic processes that characterize the distribution of random points on the real line.
Notably, Hawkes process \citep{hawkes1971point,hawkes1971spectra}, a special class of point process, has been widely investigated, partly due to its ability to capture mutual excitation among events and its mathematical tractability.
However, Hawkes process assumes that past events can only independently and additively influence the occurrence of future events, and that influence can only be excitative; these inherent limitations have restricted its modeling flexibility and render it unable to capture complex event interaction in real-world data.

As such, other more flexible models have been proposed, including the \emph{piecewise-constant conditional intensity model} (PCIM) \citep{Gunawardana2011} and its variants \citep{Weiss2013,Bhattacharjya2018}, and more recently a class of models loosely referred to as \emph{neural point processes} (NPPs) \citep{du2016recurrent,xiao2017modeling,mei2017neuralhawkes,Xiao2019}. These models, particularly NPPs, generally enjoy better predictive performance than parametric point processes, since they use more expressive machine learning models (e.g., decision trees, random forests, or recurrent neural networks) to sequentially compute the conditional intensity until next event is generated. A significant weakness of these models, however, is that they are generally uninterpretable and thus unable to provide summary statistics for determining the Granger causality among event types.

\paragraph{Granger Causality Discovery}

In his seminal paper, \citet{Granger1969} first proposed the concept of Granger causality for time series data.
Many approaches have been proposed for uncovering Granger causality for multivariate time series, including the Hiemstra-Jones test \citep{Hiemstra1994} and its improved variant \citep{Diks2006}, Lasso-Granger method \citep{Arnold2007}, and approaches based on information-theoretic measures \citep{Hlavackova-Schindler2007}.
However, as these methods are designed for the synchronous multivariate time series, they are not directly applicable to asynchronous multi-type event sequence data, since otherwise one has to discretize the continuous observation window.

\citet{Didelez2008} first established the Granger causality for event types in event sequences under the framework of marked point processes.
Later, \citet{Eichler2017} shows that Granger causality for Hawkes processes is entirely encoded in the excitation kernel functions (also called impact function).
To our best knowledge, existing research for Granger causality discovery from event sequences appears to be limited to the case of Hawkes process \citep{Eichler2017,xu2016learning,achab2017uncovering}, possibly because of this direct link between the process parameterization and Granger causality.

\paragraph{Prediction Attribution for Black-Box Models}

Prediction attribution, the task of assigning to each input feature a score for representing the feature's contribution to model prediction, has been attracting considerable interest in the field due to its ability to provide insight into predictions made by black-box models such as neural networks.
While various approaches have been proposed, there are two prominent groups of approaches: perturbation-based and gradient-based approaches.
Perturbation approaches \citep{Zeiler2014} typically comprise, first, removing, masking, or altering a feature, and then measuring the attribution score of that feature by the change of the model’s output. While perturbation-based methods are simple, intuitive, and applicable to almost all black-box models, the quality of the resultant scores is often sensitive to how the perturbation is performed. Moreover, as these methods scale linearly with the number of input features, they become computationally unaffordable for high-dimensional inputs.

In contrast, backpropagation-based methods construct the attribution scores based on the estimation of local gradients of the model around the input instance with backpropagation. The ordinary gradients, however, could suffer from a ``saturation'' problem for neural networks with activation functions that contain constant-valued regions (e.g., rectifier linear unit (ReLUs)); that is, the gradient coming into a ReLU during the backward pass is zero’d out if the input to the ReLU during the forward pass is in a constant region. One valid solution to this issue is to replace gradients with discrete gradients and use a modified form of backpropagation to compose discrete gradients into attributions, such as layer-wise relevance propagation (LRP) \citep{Bach2015} and DeepLIFT \citep{shrikumar2017learning}. Another solution, proposed by Integrated Gradient (IG) \citep{Sundararajan2017}, is to use the line integral of the gradients along the path from the input to a chosen baseline.
\citet{Sundararajan2017} show that IG satisfies many desirable properties, as detailed in Proposition~\ref{prop:IG-DeepLIFT-properties}.

It is worth mentioning that much existing work often uses the intermediate results, produced by certain intelligible neural network architecture, as the attribution scores for an input. A most notable example of such an idea is the use of attention weights induced by some attention mechanism as the importance of the input \citep{Bahdanau2015,Xu2015}.
Recently, however, there are growing concerns on the validity of attention weights being used as the explanation of neural networks \citep{jain2019attention,Serrano2019}. In particular, \citet{jain2019attention} show that across a variety of NLP tasks, the learned attention weights are frequently uncorrelated with feature importance produced by gradient-based prediction attribution methods, and random permutation of attention weights can nonetheless yield equivalent predictions.

\section{Additional Technical Details}
\label{ap-sec:techinical}

\subsection{Proof of Proposition~\ref{prop:IG-DeepLIFT-properties}}
\label{subsec:proof-IG-DeepLIFT-properties}

\begin{proof}
That both IG and DeepLIFT satisfy~\ref{P:linearity}--\ref{P:impl-invariance} has been established in \citep{Sundararajan2017}. \ref{P:fidelity-to-control} is straightforward from the definion of either method. Thus, we only prove that both methods satisfy batchability~\eqref{P:batchability} with $F(\X) \defas \sum_{i = 1}^n f(\x_i)$.

To prove that IG satisfies batchability, we first rewrite the $\IG(F, \X, \ulX)$ as follows:
\begin{align*}
 \IG(F, \X, \ulX) &= (\X - \ulX) \odot \int_{0}^1 \nabla_{\X} F \left[\ulX + \alpha (\X - \ulX)\right] d\alpha \\
 &= (\X - \ulX) \odot \int_{0}^1 \sum_{i = 1}^n \nabla_{\X} f \left[\ulx_i + \alpha (\x - \ulx_i)\right] d\alpha
 \\
 &= (\X - \ulX) \odot \int_{0}^1 \sum_{i = 1}^n \lbrace\nabla_{\x_i} f \left[\ulx_i + \alpha (\x - \ulx_i)\right] \rbrace \e_i^T d\alpha
 \\
 &= (\X - \ulX) \odot {\left[ \int_{0}^1 \nabla_{\x_i} f \left[\ulx_i + \alpha (\x - \ulx_i)\right] d\alpha \right]}_{i = 1, \ldots, m},
\end{align*}
where the second step is due to that summation and gradients are swapable, and the third step is because the gradients of different terms are separable.
Thus, we have
\begin{equation}
 {\left[\IG(F, \X, \ulX)\right]}_{:, i} = (\x_i - \ulx_i) \odot \int_{0}^1 \nabla_{\x_i} f \left[\ulx_i + \alpha (\x - \ulx_i)\right] d\alpha = \IG(f, \x_i, \ulx_i),
\end{equation}
which establishes the formula.

The proof of DeepLIFT satisfying batchability can be established in a similar way as IG.
The key part, shown in the Proposition 2 of \citep{Ancona2017}, is that the attribution scores produced by DeepLIFT for a neural-network-like function $f$, an input $\ulx$, and a baseline $\ulx$, i.e., $\DeepLIFT(f, \x, \ulx)$ can be viewed as the Hadamard product between $\x - \ulx$ and a modified gradient of $f$ at all its internal nonlinear layers. Since the last layer of $F$ is a simple linear addition of all $f(\x_i)$'s, the modified gradient of $F$ for input $\x_i$ is the same as the one of $f$ for $\x_i$. Thus, we have
\begin{equation}
  {\left[\DeepLIFT(F, \X, \ulX)\right]}_{:, i} = \DeepLIFT(f, \x_i, \ulx_i).
\end{equation}

\end{proof}

\subsection{Proof of Proposition~\ref{prop:shapley-properties}}
\label{ap:subsec:proof-shapley-properties}

We first briefly review Shapley values.
Suppose there is a team of $d$ players working together to earn a certain amount of value. The value that every coalition $U \subseteq [d]$ achieves is $v(U)$, where $v : 2^d \mapsto \bbR$ is a value function.
Shapley values, proposed by \citet{shapley1953value}, provide a well-motivated way to decide how the total earning $v([d])$ should be distributed among such $d$ players. Specifically, the Shapley value for each player $i \in [d]$ is defined as
\begin{equation}
 \phi_v(i) = \sum_{U \subseteq [d] \setminus \{i\}} \frac{(|U|! (d - |U| - 1)!)}{d!} \left[ v(U\cup \{i\}) - v(U) \right].
 \label{eq:shapley-def}
\end{equation}

For any target function $f \in \cF_d$, input $\x \in \cX$, and baseline $\ulx \in \cX$, we define a value function $v_{f, \x, \ulx}(U) \defas f(\x_{U} \sqcup \x_{\bar{U}}) $ for any $U \in [d]$, where $\x_{U} \sqcup \x_{\bar{U}}$ is the spliced data point between $\x$ and $\ulx$, defined in~\eqref{eq:spliced-data}.
Then the Shapley values $[\phi_{v_{f, \x, \ulx}}(i)]_{i \in [d]}$ can be viewed as an attribution method that provides the attribute scores for any $f$, $\x$, and $\ulx$.

Now we prove that this attribution method based on Shapley valeus satisfies all four axioms (\ref{P:linearity}--\ref{P:impl-invariance}) and the fidelity-to-control (\ref{P:fidelity-to-control}), as stated in Proposition~\ref{prop:shapley-properties}.

\begin{proof}
First, it is clear from the definition of Shapley values in~\eqref{eq:shapley-def} that $\phi_{v_{f, \x, \ulx}}(\cdot)$ satisfies linearity~\eqref{P:linearity} and implementation variance~\eqref{P:impl-invariance}.
Since \citet{shapley1953value} shows that for any value function $v$, the Shapley values $\phi_v(\cdot)$ satisfies that
\begin{equation*}
 \phi_v([d]) - \phi_v(\emptyset) = \sum_{i = 1}^d \phi_v(i),
\end{equation*}
substituting our definition of the value function $\phi_{v_{f, \x, \ulx}}(\cdot)$ into the above equation yields
\[
 f(\x) - f(\ulx) = \sum_{i = 1}^d \phi_{v_{f, \x, \ulx}}(i),
\]
which establishes the completeness \eqref{P:completeness}.
For any $i \in [d]$ and $U \subseteq [d] \setminus \{i\}$, we have
\[
 v_{f, \x, \ulx}(U \cup \{i\}) - v_{f, \x, \ulx}(U) = f(\x_{U \cup \{i\}} \sqcup \x_{\bar{U} \setminus\{i\} }) - f(\x_{U} \sqcup \x_{\bar{U}})
\]
Note that $\x_{U \cup \{i\}} \sqcup \x_{\bar{U} \setminus\{i\} }$ and $\x_{U} \sqcup \x_{\bar{U}}$ only potentially differ on the $i$-th dimension. If $f$ does not depend on the $i$-th dimension of its input or $x_i = \underline{x}_i$ (which implies $\x_{U \cup \{i\}} \sqcup \x_{\bar{U} \setminus\{i\} } = \x_{U} \sqcup \x_{\bar{U}}$), then $f(\x_{U \cup \{i\}} \sqcup \x_{\bar{U} \setminus\{i\} }) = f(\x_{U} \sqcup \x_{\bar{U}})$ and thereby $\phi_{v_{f, \x, \ulx}}(i) = 0$. Thus, $\phi_{v_{f, \x, \ulx}}(\cdot)$ satisfies null player~\eqref{P:null-player} and fidelity-to-control~\eqref{P:fidelity-to-control}.

\end{proof}

\subsection{Dyadic Gaussian Basis}
\label{ap:subsec:dyadic-gaussian-basis}

Inspired by the dyadic interval bases used by \citet{Bao2017}, we choose the basis functions $\{\psi_r(\cdot)\}_{r = 1}^R$ to be the densities for a Gaussian family $\{\cN(\mu_r, \sigma_r^2) \}_{r = 1}^R$, which we term \emph{dyadic Gaussian basis}.
The means of dyadic Gaussian basis are given by
\begin{equation}
  \mu_r = \begin{cases}
    0, & r = 1, \\
    L / 2^{R - r},  & r = 2, \ldots, R,
  \end{cases}
\end{equation}
and the standard deviations by $\sigma_r = \max(\mu_r / 3, \mu_2 / 3)$ for $r \in [R]$.
This design of basis functions reflects a reasonable inductive bias that the CIFs should vary more smoothly as the time increases. As shown in Figure~\ref{fig:dyadic-normal-example} for an example of $L=100$ and $R=5$, the first a few bases, due to their small means and variances, capture the short-term effects, whereas the last several characterize the mid/long-term effects.

\begin{figure}[!htbp]
  \centering
  \includegraphics[width=0.6\columnwidth]{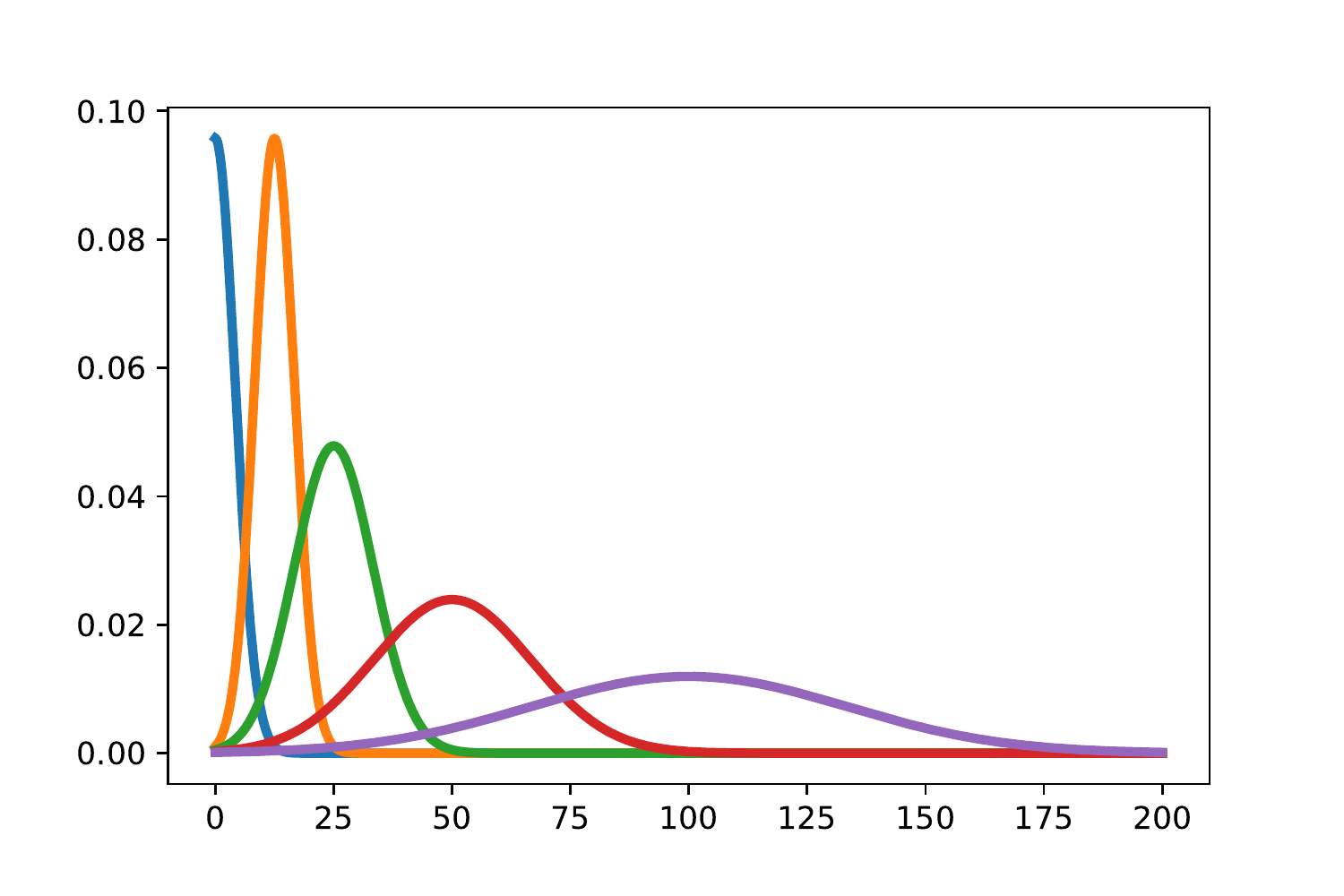}
  \caption{An example of dyadic Gaussian bases for $L = 100$ and $R = 5$.
  The first a few bases, due to their small means and variances, can capture the short-term effects, whereas the last several characterize the mid/long-term effects.
  }
  \label{fig:dyadic-normal-example}
\end{figure}

\subsection{Proof of Proposition~\ref{prop:intra-seq-batching}}
\label{ap:subsec:proof-intra-seq-batching}

\begin{proof}

We omit the index $s$ in this proof for brevity. First, we rewrite $\tdY_{k, k'}$ as
\begin{align*}
 \tdY_{k, k'}= & \sum_{i = 1}^n \sum_{j=1}^i \bbI(k_j = k') \Attr_j(f_k, \x_i, \ulx_i) \\
 = & \sum_{j = 1}^n \bbI(k_j = k') \left[ \sum_{i=j}^n \Attr_j(f_k, \x_i, \ulx_i) \right] \\
 = & \sum_{j = 1}^n \bbI(k_j = k') \left[ \sum_{i=j}^n \Attr_j(F_{k, i}, \x_n, \ulx_n) \right],
\end{align*}
where in the step step, we replace $f$ with $F$.
Since $F_{k, i}$, i.e., $ \int_{t_i}^{t_{i + 1}} \lambda_k(t') dt $, does not depend on the events before the $i$-th event, with null player~\eqref{P:null-player}, we have $\Attr_j(F_{k, i}, \x_n, \ulx_n) = 0$ for any $j < i$, which further implies that
\[
 \tdY_{k, k'} = \sum_{j = 1}^n \bbI(k_j = k') \left[ \sum_{i=1}^n \Attr_j(F_{k, i}, \x_n, \ulx_n) \right].
\]
With linearity~\eqref{P:linearity}, we have
\begin{equation*}
 \tdY_{k, k'} = \sum_{j = 1}^n \bbI(k_j = k') \Attr_j \left(\sum_{i=1}^n F_{k, i}, \x_n, \ulx_n \right),
\end{equation*}
which establishes the formula.

\end{proof}

\section{Additional Experimental Details}

\subsection{The settings for synthetic and real-world datasets.}
\label{ap:subsec:datasets}

We describe below the setup and preprocessing details for the five datasets that we consider in this paper. The statistics of these datasets are summarized in Table~\ref{tab:dataset-statistics}.
\begin{itemize} [itemsep=3pt]

 \item \textbf{\Excitation.}
 This dataset was generated by a multivariate Hawkes process, whose CIFs are of the form:
 \[
   \lambda^*_k(t) = \mu_k + \sum_{i : t_i < t } \alpha_{k, k'} \beta_{k, k'} \exp[- \beta_{k, k'} (t - t_i)].
 \]
 We set $S=1000$, $K=10$, $n_s \sim \Pois(250)$, $\mu_k \sim \Uniform(0, 0.01)$, and $\beta_{k, k} \sim \Exp(0.05)$. To generate a sparse excitation weight matrix $ \A \defas [\alpha_{k, k'}]_{k, k' \in [K]}$, we first selected all its diagonal entries and $M = 16$ random  off-diagonal entries, then generated the values for these entries from $\Uniform(0, 1)$, and finally scaled the matrix to have a spectral radius of $0.8$.

 \item \textbf{\Inhibition.}
 This dataset was generated by a multivariate self-correcting point process, whose CIFs take the form:
 \[
  \lambda^*_k(t) = \exp(\alpha_k t + \sum_{i: t_i < t} w_{k, k_i}).
 \]
 We chose $S = 1000$, $K = 10$, $n_s \sim \Pois(250)$, and $\alpha_k \sim \Uniform(0, 0.05)$. To generated a sparse weight matrix $\W = [w_{k, k'}]_{k, k' \in [K]}$, we first selected all its diagonal entries and $M = 16$ random off-diagonal entries and further generated the values for these entries from $\Uniform(-0.5, 0)$.

 \item \textbf{\Synergy.}
 This dataset was generated by a proximal graphical event model (PGEM) \citep{Bhattacharjya2018}. PGEM assumes that the CIF of an event type depends only on whether or not its parent event types (specified by a dependency graph) have occurred in the most recent history.
 We designed a local dependency structure that consists of five event types labeled as A--E. Among these event types,  type E is the outcome and can be excited by the occurrence of type A, B, or C; type A and B, only when both occurred in the most recent history, would incur a large synergistic effect on type E; type C has an isolated excitative effect on type E and does not interacts with other event types; and finally, type D does not have any excitative effect and is introduced to complicate the learning task.
 The dependency graph, together with the corresponding time windows and intensity tables, illustrated in Figure~\ref{fig:pgem-setting}.
 To add more complexity to this dataset,  we further replicated this local structure for another copy, leading to a total of $K=10$ event types. We generated $S = 1000$ event sequences with a maximum time span of $T = 1000$.

 \item \textbf{\IPTV.} We obtained the dataset from\footnote{\url{https://github.com/HongtengXu/Hawkes-Process-Toolkit/tree/master/Data}}. We further normalized the timestamps into the days and splitted long event sequences so that the length of each sequence is smaller or equal to $1000$.

 \item \textbf{\MT.} We downloaded the raw MemeTracker phrase data from\footnote{\url{https://www.memetracker.org/data.html\#raw}}. We filtered the phrase data that occurred from 2008-08-01 to 2008-09-30 and from the top-100 website domains. We further normalized the timestamps into hours and filtered out those event sequences (i.e., phrase cascades) whose lengths are not in between $3$ and $500$.

\end{itemize}

\begin{table}[!htbp]
  \centering
  \caption{Statistics for various datasets.}
  \begin{tabular}{crrrc}
	\toprule
	  \textbf{Dataset} & \textbf{$S$} & \textbf{$K$}
	& \textbf{\# of events} & \textbf{Ground truth} \\
	\midrule
	\Excitation		& 1,000		& 10		& 250,447 	& Weighted\\
	\Inhibition		& 1,000		& 10		& 250,442		& Weighted\\
	\Synergy 			& 1,000		& 10		& 178,338 	& Binary\\
	\IPTV 				& 1,869 	& 16 		& 966,338 	& N/A\\
	\MT 					& 382,014 & 100 	& 3,419,399 & Weighted\\
	\bottomrule
\end{tabular}

  \label{tab:dataset-statistics}
 \end{table}

\begin{figure}[!htbp]
  \centering
  \includegraphics[width=0.7\textwidth]{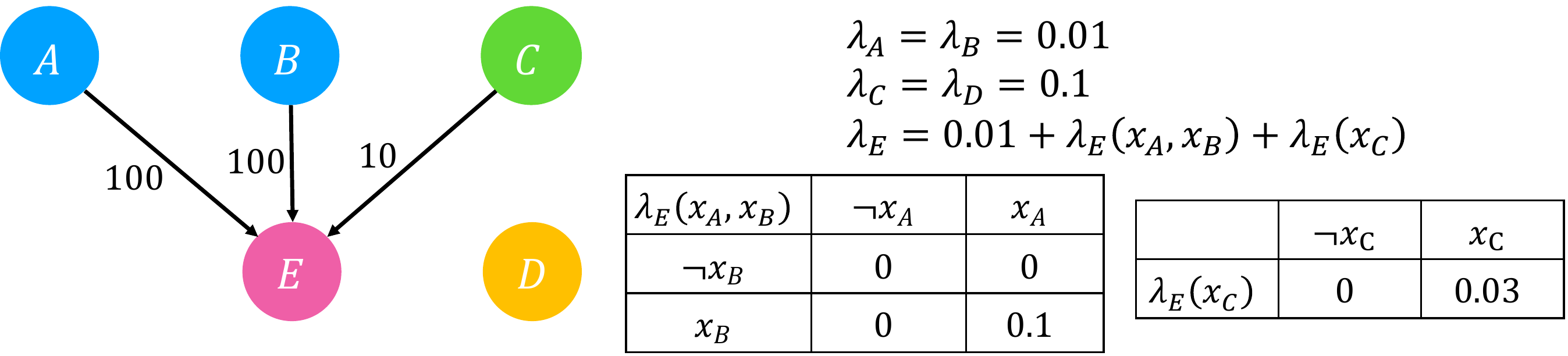}
  \caption{The dependency graph, time windows, and intensity tables for the PGEM used in generating the \Synergy dataset.}
  \label{fig:pgem-setting}
\end{figure}

\subsection{Implementation Details and Hyperparameter Configurations for Various Methods}
\label{ap:subsec:method-config}

For \method, the $\SeqEnc(\cdot)$ and the $\balpha(\cdot)$ were implemented by a single-layer GRU and a two-layer fully connected network with skip connections, respectively.
The dimension of event type embeddings was fixed to $64$, and the number of hidden units for GRU was set to be $64$ for synthetic datasets and $128$ for real-world datasets.
The number of basis functions $R$ and the maximum mean $L$ were chosen by a rule of thumb such that $\mu_2$ and $\mu_R$ are of the same scale as the $50$th and the $99$th percentiles of the inter-event elapsed times, respectively.
The optimization was conducted by Adam with an initial learning rate of $0.001$.
A hold-out validation set consisting of $10\%$ of the sequences of the training set was used for model selection; the model snapshot that attains the smallest validation loss was chosen.
As events sequence lengths vary greatly on two real-world datasets, in constructing mini-batches for both training and inference, we adopted the bucketing technique to reduce the amount of wasted computation caused by padding.
Finally, the line integral of IG, defined in~\eqref{eq:ig-defn}, was approximated by $20$ steps for \MT and $50$ steps for other datasets; a smaller number of steps, although may result in certain lose of accuracy, allows for a larger batch size and thus shorter execution time for attribution.

For the Hawkes process-based baselines---HExp, HSG, and NHPC---their implementation was provided by the package \texttt{tick} \citep{bacry2017tick}. The most relevant hyperparameters for each method were tuned by cross-validation.

As there is no publicly available codes for RPPN, we implemented it with our best effort. Its overall settings for architecture and optimization is similar to the ones for \method.

\subsection{Platform and Runtime}

All experiments were conducted on a server with a 16-core CPU, 512G memory, and two Quadro P5000 GPUs.
On the largest dataset, \MT, the total runtime for \method was less than 3 hours, including both training and computing the Granger causality statistic.


\subsection{Qualitative Analysis on \MT}
\label{ap:subsec:qualitative-analysis-MT}

Since there are too many event types in \MT, instead of a heat map, we visualize the causality matrix as a graph and show in Figure~\ref{fig:granger-causality-matrix-MT-1} and Figure~\ref{fig:granger-causality-matrix-MT-2} the top-two communities of that graph, where the directed edges denote the estimated Granger causality between pairs of domains.\footnote{The graph visualization and community detection were performed using the software \texttt{Gephi}.}
In Figure~\ref{fig:granger-causality-matrix-MT-1}, the domain \texttt{news.google.com} centers in the middle and is pointed by many sites, which is unsurprising because Google News aggregates articles from other publishers and websites. Our method also correctly identifies other major ``information-consuming'' domains such as \texttt{bogleheads.org}, an active forum for investment-related Q\&A.
In Figure~\ref{fig:granger-causality-matrix-MT-2}, the then very popular social networking website \texttt{blog.myspace.com} sits in the center of the community.
Our method also identifies credible excitative relationships between the subdomains of \texttt{craigslist.org}, a mega-website that hosts classified, local advertisements.

\begin{figure*}[!tbp]

  \centering
  \subfloat[\MT (Community 1)]{
  \includegraphics[width=0.45\textwidth]{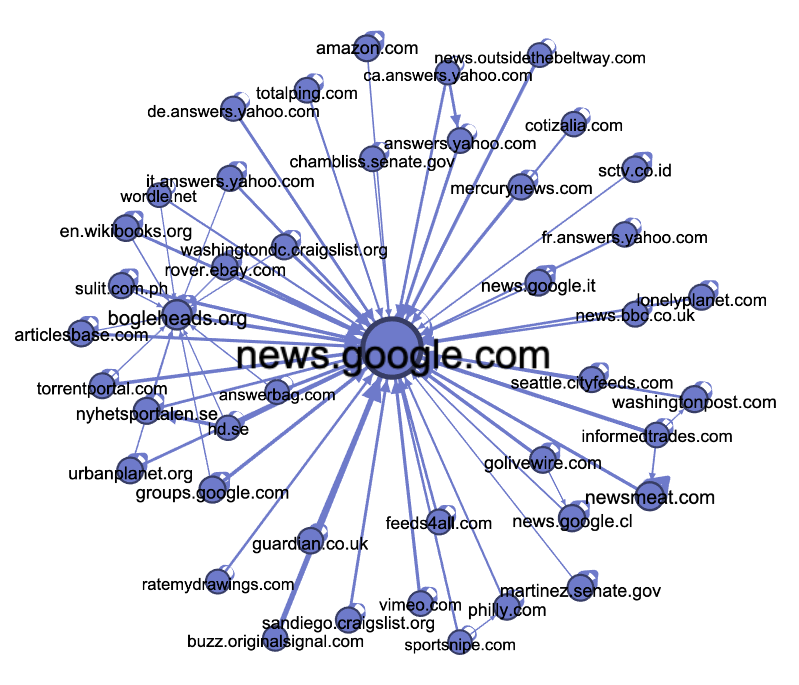}
  \label{fig:granger-causality-matrix-MT-1}
  }
  \subfloat[\MT (Community 2)]{
  \includegraphics[width=0.45\textwidth]{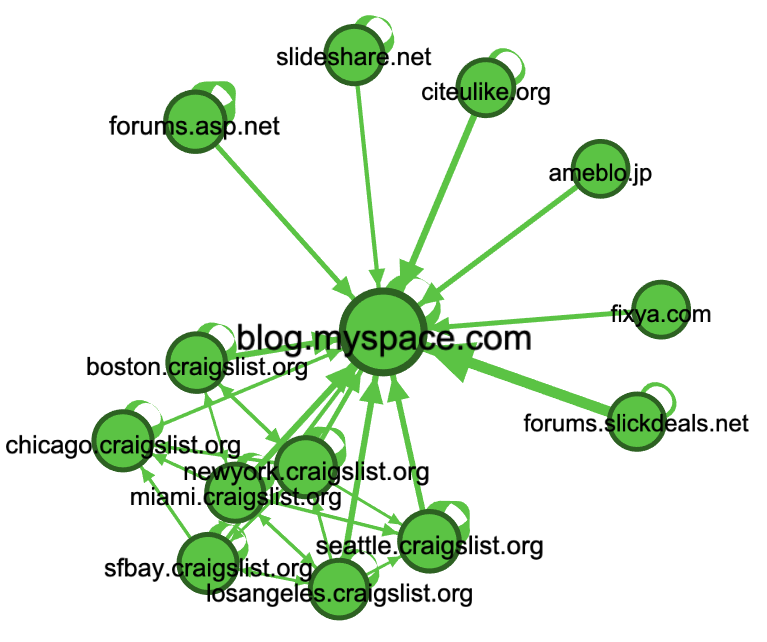}
  \label{fig:granger-causality-matrix-MT-2}
  }
  \caption{The top-two communities of the estimated Granger causality statistic matrices on \MT (Figure~\ref{fig:granger-causality-matrix-MT-1} \&~\ref{fig:granger-causality-matrix-MT-2}). Better viewed on screen.}

  \label{fig:granger-causality-matrix}
\end{figure*}

\section{A Primer on Measure and Probability Theory}
\label{ap:sec:measure-theory}

In this section, we review some of the basic definitions of in measure theory, which may help the understanding of the definition of Granger causality for multivariate point processes. Most of the content in this section were adapted based on primarily based on the Chapter 1 of \citep{durrett2019probability} and the Appendix 3 of \citep{Daley2003},

Let $\Omega$ be a set of ``outcomes'' and $\cF$ a nonempty collection of subsets of $\Omega$. The set $\cF$ is \textbf{$\sigma$-algebra} of $\Omega$, if it is closed under complement and countable unions; that is,
\begin{enumerate}
  \item if $A \in \cF$, then $ \Omega \setminus A \in \cF$, and
  \item if $A_i \in \cF$ is a countable sequence of sets, then $\cup_i A_i \in \cF$.
\end{enumerate}
With these two conditions, it's easy to see that $\sigma$-algebra is also closed under arbitrary (possibly uncountable) intersections.
From this, it follows that given a nonempty set $\Omega$ and a collection of $\cA$ of subsets of $\Omega$, there is a smallest $\sigma$-algebra containing $\cA$; we denote such smallest $\sigma$-algebra by $\sigma(\cA)$.
One particular $\sigma$-algebra is of particular interst---\textbf{Borel $\sigma$-algebra}; that is, the smallest $\sigma$-algebra containing all open sets in $\bbR^d$, denoted by $\cR^d$.
Specifically, let $\cS_d$ be the empty set plus all sets of the form $(a_1, b_1] \times \cdots (a_d, b_d] \subset \bbR^d$, where $-\infty \leq a_i < b_i \leq \infty$, then $\cR^d = \sigma(\cS^d)$. The superscript $d$ is dropped when $d = 1$.

A pair $(\Omega, \cF)$, in which $\Omega$ is a set and $\cF$ is a $\sigma$-algebra of $\Omega$, is called a \textbf{measurable space}. A \textbf{measure} defined on $(\Omega, \cF)$ is a nonnegative countably additive set function; that is a function $\mu: \cF \mapsto \bbR$ with
\begin{enumerate}
  \item $\mu(A) \geq \mu(\emptyset) = 0$ for all $A \in \cF$, and
  \item if $A_i \in \cF$ is a countable sequence of disjoint sets, then
  \begin{equation*}
    \mu(\bigcup_i A_i) = \sum_i \mu(A_i )
  \end{equation*}
\end{enumerate}
If $\mu(\Omega) = 1$, we call such a $\mu$ a \textbf{probability measure}.
The triplet $(\Omega, \cF, \mu)$ is called a \textbf{measure space}, and a \textbf{probability space} if $\mu$ is a probability measure.

Given a probability space $(\Omega, \cF, \mu)$, a real-valued function $X$ defined on $\Omega$ is said to be a \textbf{random variable} if for every Borel set $B \in \cR$ we have $X^{-1}(B) \defas \{\omega: X(\omega) \in B \} \in \cF$; in another words, $X$ is \textbf{$\cF$-measuable}.
A \textbf{stochastic process} is a collection of random variables $\{X_i\}_{i \in \cI}$ defined on a common probability space and indexed by a \textbf{index set} $\cI$. In most cases, the index set can be positive numbers $\bbN_+$, or real line $\bbR_+$.
A \textbf{filtration} is a sequence of $\sigma$-algebras, denoted by $\{\cF_i\}_{i \in \cI}$, if $\cF_j \subseteq \cF_i$ if $j \leq i$ and $i, j \in \cI$.
Given a stochastic process $\{X_i\}_{i \in \cI}$ defined on $(\Omega, \cF, \mu)$, the \textbf{natural filtration} of $\cF$ with respect to the process is given by
\begin{equation}
  \cH_i \defas \sigma \left(\{ X_j^{-1}(B) | j \in \cI, j \leq i, B \in \cR \} \right).
  \label{eq:natural-filtration}
\end{equation}
It is in a sense that the simplest filtration filtration available for studying the given: all information concerning the process, and only that information, is avaiable in the natural filtration. Thus, the natural filtration $\cH_i$ can be often be viewed as the ``\textbf{history}'' of the subprocess $\{X_j\}_{j \leq i, j \in I}$. Note that sometimes the definition in~\eqref{eq:natural-filtration} is simply written as $\cH_i \defas \sigma \left( \{ X_j | j \in \cI, j \leq i\} \right)$.

A \textbf{point process} $\{T_i\}_{i\geq 1}$ is a real-valued stochastic process indexed on $N_+$ such that $T_i \leq T_{i + 1}$ almost surely. Each random variable is generally viewed as the arrival timestamp of an event.
For each point process, one can define a continuously indexed stochastic process associated with it called \textbf{counting process}, as $N(t) \defas \sum_{i = 1}^{\infty} \bb1( T_i \leq  t)$. From this definition, it is easily seen that every realization of a counting process is a c\`adl\`ag (i.e.\ right continuous with left limits) step function, and that a counting process $N(t)$ equivalently defines a point process, as one can recover the event timestamp by $T_i = \inf \{t \geq 0: N(t) = i \}$. Due to this equivalence, the phrases point process and counting process, as well as their notation, $\{T_i\}_{i \in \bbN_+}$ and $N(t)$, are often used interchangably in the literature.
A $K$-dimensional \textbf{multivariate point process (MPP)}  is a coupling of $K$ point/counting process $\N(t) = [N_1(t), N_2(t),\ldots, N_K(t)]$. A realization of a multivariate point process is a multi-type event sequence, $\{(t_i, k_i)\}_{i \in \bbN_+}$, where $t_i$ indicates the event timestmap of the $i$-th event, and the $k_i$ indicates which dimension the $i$-th event comes from (often interpreted as event type).

The most common way to define an MPP is through a set of \textbf{conditional intensity functions (CIFs)}, one for each event type.
Specifically, let $\cH(t) \defas \sigma( \{N_k(s) | k \in [K], s < t\})$ for any $t$ be the natural filtration of MPP and let $\cH(t-) \defas \lim_{s \uparrow t} \cH(s)$ the CIF for event type $k$ is defined as the expected instantaneous event occurrence rate conditioned on natural filtration, i.e.,
\begin{equation*}
 \lambda^*_k(t) \defas \lim_{\Delta t \downarrow 0} \frac{\bbE [N_k(t + \Delta t) - N_k (t) | \cH(t)]}{\Delta t},
\end{equation*}
where the use of the asterisk is a notational convention to emphasize that intensity $\lambda^*_k(t)$ must be $\cH(t)$-measurable for every $t$.

Finally, for any $\cK \subseteq [K]$, denote by $\cH_{\cK}(t)$ the natural filtration expanded by the sub-process ${\{N_k(t)\}}_{k \in \cK}$, i.e., $\cH_{\cK}(t) = \sigma( \{N_k(s) | k \in \cK, s < t\})$, and further write $\cH_{-k}(t) = \cH_{[K] \setminus \{k\}}(t)$ for any $k \in [K]$.  For a $K$-dimensional MPP, event type $k$ is \textbf{Granger non-causal} for event type $k'$ if $\lambda^*_{k'}(t)$ is $\cH_{-k}(t)$-measurable for all $t$.
This definition amounts to saying that a type $k$ is Granger non-causal for another type $k'$ if, conditioned on the history of events other than type $k$, the future $\lambda^*_{k'}(t)$ does not depend on the historical events of type $k$ at any time. Otherwise, type $k$ is said to be \emph{Granger causal} for type $k$.

\end{document}